\documentclass{article}

\usepackage[preprint]{neurips_2025}

\usepackage{mathrsfs}

\usepackage[toc]{appendix}
\usepackage[utf8]{inputenc} 
\usepackage[T1]{fontenc}    
\usepackage{hyperref}       
\usepackage{url}            
\usepackage{booktabs}       
\usepackage{amsfonts}       
\usepackage{nicefrac}       
\usepackage{microtype}      
\usepackage{xcolor}         
\usepackage{listings}
\usepackage{matlab-prettifier}
\usepackage{microtype}      

\usepackage{bm}
\usepackage{multirow}
\usepackage{ulem}
\usepackage{bbm}
\usepackage[table]{xcolor}
\usepackage{comment}
\usepackage{graphicx} 
\usepackage{subfigure}
\usepackage{tabularx}
\usepackage{booktabs}
\usepackage{wrapfig}
\usepackage{amsthm}
\usepackage{algorithmic}
\usepackage[ruled,linesnumbered]{algorithm2e}

\definecolor{pos}{HTML}{5727b0}
\definecolor{neg}{HTML}{b02780}
\definecolor{pl}{HTML}{b0273c}
\definecolor{product_dt}{HTML}{1F77B4}
\definecolor{euclidean_dt}{HTML}{FF7F0E}
\definecolor{tangent_dt}{HTML}{2CA02C}
\definecolor{knn}{HTML}{D62728}
\definecolor{perceptron}{HTML}{9467BD}
\definecolor{darkgreen}{HTML}{008000}
\definecolor{mlp}{HTML}{9467BD}
\definecolor{kgcn}{HTML}{8C564B}

\usepackage{caption}

\usepackage{graphicx}
\usepackage{amsmath}
\theoremstyle{plain}
\newtheorem{theorem}{Theorem}[section]

\theoremstyle{definition}

\theoremstyle{remark}

\title{The Post-GCN Decade Revisited: Curvature-Stratified Evaluation of Relational Learning}

\author{%
    \textbf{Shuo Wang}\textsuperscript{1,2}\thanks{These authors contributed equally.}
    \quad \hspace{0.2cm}
    \textbf{Xiangyu Wang}\textsuperscript{1}\footnotemark[1]
    \quad \hspace{0.2cm}
    \textbf{Quanxin Wang}\textsuperscript{1}
    \quad \hspace{0.2cm}
    \textbf{Bailin Wu}\textsuperscript{1}
    \quad \hspace{0.2cm}
    \textbf{Bokui Wang}\textsuperscript{1}
    \\
    \textbf{Shunyang Huang}\textsuperscript{1}
    \quad \hspace{0.2cm}
    \textbf{Boyan Deng}\textsuperscript{1}
    \quad \hspace{0.2cm}
    \textbf{Haonan Liu}\textsuperscript{1}
    \quad \hspace{0.2cm}
    \textbf{Ruiyi Fang}\textsuperscript{3}
    \quad \hspace{0.2cm}
    \textbf{Zhenxiang Xu}\textsuperscript{1,4}
    \quad 
    \\
    \textbf{Boyu Wang}\textsuperscript{3}
    \quad \hspace{0.2cm}
    \textbf{Zhao Kang}\textsuperscript{1}\thanks{Corresponding author.}
    \\
    \textsuperscript{1} University of Electronic Science and Technology of China
    \quad
    \textsuperscript{2} Tsinghua University
    \\
    \textsuperscript{3} Western University
    \quad
    \textsuperscript{4} Zhejiang University
    \\
    \texttt{runner21st@gmail.com}
    \quad
    \texttt{zkang@uestc.edu.cn}
}

\begin{document}

\maketitle
\begin{abstract}
Current evaluation practices in relational learning rely heavily on flat leaderboards that average performance across heterogeneous datasets, implicitly assuming a uniform underlying structure. We show that this assumption introduces systematic bias: it obscures geometry-dependent performance variations and can lead to misleading conclusions about model generalization. In this work, we identify intrinsic geometry as a key latent factor governing model effectiveness. We demonstrate that conventional aggregated metrics mask critical performance trade-offs that only become visible when datasets are stratified by their geometric properties. To address this issue, we introduce a curvature-stratified evaluation framework that partitions datasets into positive, negative, and near-zero curvature regimes. Our benchmark evaluates 18 representative models—including Graph Convolutional Networks (GCNs), Graph Foundation Models (GFMs), and tabular learning methods—across 14 datasets. We find that model rankings are highly stable within each curvature regime but shift significantly across regimes, indicating that performance is fundamentally geometry-dependent rather than universally transferable. Notably, we identify regimes where GFMs offer diminishing returns compared to geometry-aligned GNNs. Based on these findings, we propose a geometry-aware evaluation protocol that yields more reliable and interpretable comparisons than standard aggregated benchmarks. We release all code, curvature-stratified dataset splits, and evaluation tools to support reproducible and rigorous assessment of future relational learning methods. Code and datasets are provided in our project homepage: \url{https://sirbabbage.github.io/CurvBench_HOME/}.

\end{abstract}
\section{Introduction}
Relational learning seeks to model data where instances are interconnected through explicit or implicit dependencies \citep{fey2023relational,robinson2024relbench}. This structure is ubiquitous, spanning citation networks and molecular graphs to tabular schemas linked by foreign-key joins \citep{morris2020tudataset,wu2018moleculenet,kim2024carte}. Despite their disparate origins, these settings share a fundamental abstraction: relations induce paths and higher-order connectivity that define the data's underlying topology \citep{bronstein2021geometric,ollivier2009ricci}. This perspective allows us to treat both graph-structured and table-derived data as instances of a common finite-metric space, amenable to unified relational analysis \citep{robinson2024relbench}. 

Over the past decade, Graph Convolutional Networks (GCNs) have emerged as the dominant paradigm, leveraging message-passing to capture neighborhood dependencies \citep{kipf2017semi,hamilton2017inductive,velickovic2018graph,schlichtkrull2018modeling}. More recently, the field has moved toward high-capacity Graph Foundation Models (GFMs)~\citep{wang2025graphfoundation,zhao2024all,yu2024textfree,wang2025multidomain} and non-Euclidean architectures \citep{nickel2017poincare,chami2019hyperbolic}. This rapid expansion has been fueled by large-scale benchmarks that provide a standardized arena for competition \citep{you2020design}. However, these benchmarks almost exclusively rely on "flat" leaderboards: performance is averaged across heterogeneous datasets, implicitly assuming that "relational data" is a monolithic category with a uniform underlying structure \citep{wu2018moleculenet,hu2020open,dwivedi2023benchmarking}.

We argue that this assumption is fundamentally flawed and has led to a systematic evaluation bias. Relational datasets differ not only in scale or sparsity but in their intrinsic geometry—the latent curvature that governs how information propagates and how representations collapse \citep{ollivier2009ricci,forman2003bochner,ni2015ricci,chami2019hyperbolic}. By aggregating results across geometrically distinct datasets (e.g., merging tree-like citation networks with grid-like molecules), current evaluation protocols marginalize over structural differences \citep{wu2018moleculenet,hu2020open}. This obscures systematic performance trade-offs, making it impossible to discern whether a model's success is due to architectural superiority or a lucky alignment with the dominant geometry of the benchmark's dataset mix \citep{dwivedi2023benchmarking,you2020design}.

While a growing body of work has introduced hyperbolic, spherical, and mixed-curvature models to match specific data topologies, the evaluation paradigm has remained stagnant \citep{nickel2017poincare,ganea2018hyperbolic,liu2019hyperbolic,skopek2020mixed}. Models designed for specialized geometric regimes are still judged by global averages that hide their true utility \citep{bachmann2020constant,gu2019learning,xiong2022pseudo,grover2025spectro}. This creates a geometric mismatch between model design and model assessment: we design for curvature, yet we evaluate for the "average".

In this work, we bridge this gap by showing that intrinsic geometry is a key latent factor governing model effectiveness. We demonstrate that conventional aggregated metrics mask critical performance shifts that only become visible when datasets are stratified by their geometric properties. We find that while model rankings are remarkably stable within a homogeneous geometric regime, they often flip when moving across regimes. Notably, we identify structural conditions where emerging GFMs offer diminishing returns compared to simpler, geometry-aligned GNNs.

To address this, we introduce \textsc{CurvBench}, a curvature-stratified audit framework for relational learning. \textsc{CurvBench} partitions 14 diverse datasets into positive, negative, and near-zero curvature regimes. By evaluating 18 representative models across these strata, we shift the focus from a single global leaderboard to regime-conditioned diagnostics. Our contributions are as follows:

\begin{itemize}
\item We develop a framework based on curvature statistics (mean and skewness) to categorize both graph and tabular datasets into geometrically meaningful regimes.
\item We evaluate 18 methods—spanning GCNs, GFMs, and tabular models—revealing that performance is fundamentally geometry-dependent rather than universally transferable.
\item We formalize a regime-stratified methodology that replaces flat leaderboards with geometry-aware comparisons, providing more reliable and interpretable insights for model selection.
\item We release all code, curvature-stratified splits, and diagnostic tools to support a more rigorous and structurally-aware assessment of future relational learning research.
\end{itemize}

\section{Graph Curvature}
\label{sec:preliminaries}

\paragraph{Graphs as finite metric spaces.}
Let \(G=(V,E,\mathbf X)\) be an attributed graph with adjacency matrix
\(\mathbf A\). We regard \(G\) as a finite metric space \((V,d_G)\), where
\(d_G(u,v)\) is the shortest-path distance between nodes \(u\) and \(v\).
For disconnected graphs, all quantities are computed over finite-distance pairs.
Let
\[
    \mathcal N(m)=\{v\in V:(m,v)\in E\},
    \qquad
    \operatorname{diam}_f(G)
    =
    \max_{u,v:\,d_G(u,v)<\infty} d_G(u,v).
\]

\paragraph{Midpoint curvature residual.}
For a center node \(m\), we approximate a local metric section by an unordered
neighbor pair \(\{b,c\}\subset\mathcal N(m)\), and use an anchor node
\(a\neq m\) to probe deviation from Euclidean midpoint geometry \citep{burago2001course,bridson1999metric}. For every valid
quadruple \((a,b,c;m)\), we define
\[
    \Delta_G(a,b,c;m)
    =
    d_G(a,m)^2
    +
    \frac14 d_G(b,c)^2
    -
    \frac12
    \left(
        d_G(a,b)^2+d_G(a,c)^2
    \right),
\]
and the normalized midpoint curvature residual
\begin{equation}
\label{eq:metric-sectional-score}
    \xi_G(a,b,c;m)
    =
    \frac{\Delta_G(a,b,c;m)}{2d_G(a,m)}.
\end{equation}
In Euclidean midpoint geometry, $\xi_G = 0$. Negative values signify thinner-than-Euclidean local triangles and geodesic divergence, whereas positive values indicate fatter-than-Euclidean local geometry. Consequently, the sign and distribution of $\xi_G$ provide a robust discrete signal for identifying the local curvature regime \citep{ollivier2009ricci,forman2003bochner}.

\paragraph{Node-level curvature.}
Let
\[
    \mathcal P_m
    =
    \bigl\{\{b,c\}:b,c\in\mathcal N(m),\, b<c\bigr\},
    \qquad
    \mathcal A_m
    =
    \{a\in V\setminus\{m\}:0<d_G(a,m)<\infty\}.
\]
The raw curvature estimate of node \(m\) is
\begin{equation}
\label{eq:node-curvature-raw}
    \widehat{\kappa}_G(m)
    =
    \frac{1}{|\mathcal P_m|}
    \sum_{\{b,c\}\in\mathcal P_m}
    \frac{1}{|\mathcal A_m|}
    \sum_{a\in\mathcal A_m}
    \xi_G(a,b,c;m),
\end{equation}
with \(\widehat{\kappa}_G(m)=0\) when \(|\mathcal P_m|=0\). For cross-dataset
comparison, we use the diameter-normalized relative curvature
\begin{equation}
\label{eq:relative-node-curvature}
    \kappa_G(m)
    =
    \frac{\widehat{\kappa}_G(m)}{\operatorname{diam}_f(G)}.
\end{equation}

\paragraph{Graph-level curvature distribution.}
CURVBENCH represents each graph by the empirical distribution of its node-level
relative curvature values $    \mu_G
    =
    \frac{1}{|V|}
    \sum_{m\in V}\theta_{\kappa_G(m)},$
where \(\theta_x\) denotes the Dirac measure, i.e., a unit point mass at value
\(x\). Thus, \(\mu_G\) is the empirical distribution of
\(\{\kappa_G(m):m\in V\}\). We summarize \(\mu_G\) by its mean curvature $\bar{\kappa}(G)
    =
    \frac{1}{|V|}
    \sum_{m\in V}\kappa_G(m)$ and by its curvature skewness. Let


\begin{equation}
\label{eq:curvature-std}
    \sigma_\kappa(G)
    =
    \left(
        \frac{1}{|V|}
        \sum_{m\in V}
        \left(
            \kappa_G(m)-\bar{\kappa}(G)
        \right)^2
    \right)^{1/2}.
\end{equation}
The curvature skewness is the third standardized central moment \citep{joanes1998comparing}
\begin{equation}
\label{eq:curvature-skewness}
    \gamma_\kappa(G)
    =
    \begin{cases}
    \displaystyle
    \frac{1}{|V|}
    \sum_{m\in V}
    \left(
        \frac{
            \kappa_G(m)-\bar{\kappa}(G)
        }{
            \sigma_\kappa(G)
        }
    \right)^3,
    & \sigma_\kappa(G)>0,\\[2.0ex]
    0,
    & \sigma_\kappa(G)=0.
    \end{cases}
\end{equation}
While $\bar{\kappa}(G)$ quantifies the average signed curvature, $\gamma_\kappa(G)$ captures the asymmetry of the curvature distribution. Specifically, positive values denote a right-skewed tail of high-curvature nodes, negative values indicate a left-skewed tail of low-curvature nodes, and values near zero suggest a balanced distribution. Detailed code and descriptions for metrics are given in Appendix~\ref{app:metrics}. 


\section{The Setup of \textsc{CurvBench}}




\subsection{Datasets}
\begin{table*}[h]
    \centering
    \caption{Statistics of natural graph datasets.}\label{natgraph}
    \resizebox{0.95\textwidth}{!}{
    \begin{tabular}{ccccccccccc}
    \toprule
    Regime & Dataset & Domain  & \#Nodes & \#Edges & \#Homophily & \#Avg Degree & \#Features & \#Classes & \#Mean Curv \(\bar{\kappa}(G)\) & \#Skewness \(\gamma_\kappa(G)\) \\
    \midrule
    \multirow{3}{*}{\textcolor{orange!75!black}{Near-zero}}
    & Cora~\cite{Yang2022} & Citation & 2,708 & 5,278 & 0.8100 & 3.90 & 1,433 & 7 & 0.00749 & 0.08401\\
    & Citeseer~\cite{Yang2022} & Citation   & 3,327 & 4,552 & 0.7355 & 2.74 & 3,703 & 6 & 0.00222 & 0.38363\\
    & PubMed~\cite{Yang2022} & Citation   & 19,717 & 44,324 & 0.8024 & 4.50 & 500 & 3 & 0.00678 & 0.43122\\
    \midrule
    \multirow{3}{*}{\textcolor{green!60!black}{Positive}}
    & Cornell~\cite{Pei2020} & Webpage/WebKB   & 183 & 298 & 0.1309 & 1.63 & 1,703 & 5 & 0.01050 & 0.81561\\
    & Airport~\citep{chami2019hyperbolic} & Transportation & 7,543 & 18,508 & 0.4289 & 4.91 & 7,543 & 4 & 0.00213 & 1.33127\\
    & Actor~\cite{Pei2020} & Wikipedia   & 7,600 & 30,019 & 0.2188 & 3.95 & 932 & 5 & 0.12039 & 1.30001\\
    \midrule
    \multirow{3}{*}{\textcolor{red}{Negative}}
    & Disease~\citep{chami2019hyperbolic} & Epidemiological   & 1,044 & 1,042 & 0.8752 & 0.998 & 1,000 & 2 & -0.00335 & -1.48057\\
    & Telecom~\cite{Zhou2022} & Telecommunication   & 41,143 & 41,424 & 0.5620 & 1.01 & 240 & 3 & -1.14371 & -11.82744\\
    & CS\_Phds~\cite{nr} & Academic/Social   & 1,025 & 1,043 & 0.2819 & 2.04 & 16 & 4 & -0.00301 & -1.53958\\
    \bottomrule
    \end{tabular}
    }
\end{table*}

\begin{table}[h]
    \centering
    \caption{Statistics of table-derived datasets.}\label{tabdata}
    \resizebox{0.95\textwidth}{!}{
    \begin{tabular}{cccccccccccc}
    \toprule
    Dataset& Domain & \#Orig. Tables & \#Rows & \#Columns & \#Nodes & \#Edges & \#Avg Degree & \#Features & \#Classes & \#Mean Curv \(\bar{\kappa}(G)\) & \#Skewness \(\gamma_\kappa(G)\) \\
    \midrule
    Carcinogenesis~\citep{muggleton1997predictive} & Medicine & 6 & 27,570 & 23 & 28,027 & 8,982 & 0.6410 & 300 & 3 & 0.00034 & 9.42658 \\
    Hepatitis& Medicine & 7 & 12,927 & 26 & 12,927 & 13,016 & 2.0138 & 300 & 3 & 0.00024 & 4.21239 \\
    PTE~\cite{mutlu2016policy}& Medicine & 38 & 29,762 & 76 & 29,850 & 18,805 & 1.2600 & 300 & 3 & 0.00031 & 9.74080 \\
    Toxicology& Medicine & 4 & 49,239 & 11 & 49,813 & 18,267 & 0.7334 & 300 & 3 & 0.00021 & 12.06911 \\
    F1~\cite{robinson2024relbench} & Sports & 9 & 97,606 & 77 & 97,606 & 192,560 & 3.9457 & 300 & 40 & 1.11301 & -2.26907 \\
    \bottomrule
    \end{tabular}
    }
\end{table}
CURVBENCH evaluates 14 relational datasets, including nine natural graphs and five table-derived graphs, spanning diverse scales, homophily levels, and application domains. Geometric descriptors—mean curvature $\bar{\kappa}(G)$ and curvature skewness $\gamma_\kappa(G)$— characterize the underlying structural regimes. This distinction is crucial, as near-zero mean curvature alone can obscure highly asymmetric curvature distributions. Based on these metrics, we partition datasets in Table \ref{natgraph} into three regimes: (i) \emph{near-zero geometry}, where $|\bar{\kappa}(G)| < 0.01$ and $|\gamma_\kappa(G)| < 0.5$; (ii) \emph{positively curved}; and (iii) \emph{negatively curved}. Table-derived graphs in Table \ref{tabdata} exhibit a geometry pattern that is not captured by mean
curvature alone. Four medical datasets have nearly zero mean curvature,
\(\bar{\kappa}(G)\in[0.00021,0.00034]\), but extremely positive skewness \(\gamma_{\kappa}(G)\in[4.21,12.07]\), indicating strong positive curvature
tails despite an apparently flat average profile. This curvature-aware taxonomy enables \emph{stratified evaluation}, allowing model performance to be analyzed through geometry-conditioned inductive biases rather than relying on aggregate rankings. 


\subsection{Compared models}
\label{sec:compared_models}

We evaluate 18 representative models grouped by their geometric inductive biases, enabling performance to be interpreted through curvature-conditioned behavior rather than a single aggregate leaderboard. 

\textbf{Geometry-agnostic and Euclidean methods.}
This group includes MLP as a feature-only baseline, along with GCN, GAT, and GraphSAGE as standard flat-space message passing architectures, and PCNet for Euclidean spectral filtering \citep{kipf2017semi, velickovic2018graph, hamilton2017inductive, li2024pcconv}. These models serve as a reference regime, capturing scenarios where flat aggregation and feature separability are sufficient, particularly in near-zero-curvature settings.

\textbf{Hyperbolic methods.}
HGNN, HGCN, HAT, and HyboNet embed graph representations in negatively curved spaces via Riemannian message passing, hyperboloid convolution, or hyperbolic attention \citep{liu2019hyperbolic, chami2019hyperbolic, zhang2021hyperbolic, chen2022fully}. They test whether fixed negative-curvature inductive biases are advantageous for hierarchical or tree-like structures.

\textbf{Flexible-geometry methods.}
QGCN and CUSP relax the assumption of a single global geometry by employing pseudo-Riemannian metrics or mixed-curvature spectral filtering \citep{xiong2022pseudo, grover2025spectro}. These models are designed for graphs with heterogeneous local structure, where multiple geometric regimes may coexist.

\textbf{Adaptive Riemannian method.}
GraphMoRE dynamically assigns nodes to geometry-specific experts through topology-aware gating, constructing personalized mixed-curvature representations \citep{guo2025GraphMoRE}. This setting tests whether local curvature heterogeneity is better handled through adaptive routing rather than fixed global geometry.

\textbf{Graph foundation models (GFMs).}
We include recent GFMs to examine whether multi-domain pre-training mitigates or preserves curvature-dependent behavior. GCOPE \citep{zhao2024all}, MDGPT \citep{yu2024textfree}, and SAMGPT \citep{yu2025samgpt} rely on token-based or prompt-based adaptation; MDGFM \citep{wang2025multidomain} incorporates structure learning; GraphGluing \citep{sun2026graphglue} performs manifold-level alignment; and SA2GFM \citep{shi2026sa2gfm} injects hierarchical structural semantics. 

Together, this suite enables a systematic comparison between classical geometric models and modern pretrain-and-adapt approaches under a unified curvature-stratified evaluation protocol.

\subsection{Implementation details}

\paragraph{Dataset settings for GFMs.}
We follow the standard evaluation protocol used in prior GFM studies. Specifically, we adopt a 3-fold split in which six datasets are used for pre-training and the remaining three for tuning and evaluation. To ensure geometric diversity, the pre-training set includes two datasets from each curvature regime. Detailed dataset splits and configurations are provided in Appendix~\ref{datasplits}.

\paragraph{Hyperparameter optimization.}
To ensure fair and reproducible comparisons, we adopt default hyperparameters from the original implementations whenever available, and perform grid search over key hyperparameters within predefined ranges for each method. For each dataset–task pair, we report the best-performing configuration based on validation performance. All results are averaged over five random seeds, and we report the mean along with standard deviation. More implementation details are summarized in Appendix~\ref{implement}.

\section{Motivation and Theoretical Grounding}
\label{sec:theory}

CURVBENCH is motivated by two fundamental questions. 
First, why should graph curvature influence the suitability of a representation geometry? 
Second, if curvature systematically affects model behavior, how should models be compared across heterogeneous regimes?

Standard evaluation protocols implicitly assume that a single aggregate leaderboard provides a meaningful comparison. Formally, a model \(M\) is evaluated under a benchmark mixture as
\[
    R_{\pi,t}(M)=\sum_{r\in\mathcal R}\pi_r\, R_{r,t}(M),
\]
where \(r\) indexes curvature regimes in $\mathcal{R}$, \(t\) denotes the task, and \(\pi_r\) are mixture weights over regimes. \(R_{r,t}(M)\) denotes the average evaluation score of model \(M\) on task \(t\) within curvature regime \(r\).

However, this aggregation obscures regime-dependent behavior. If two models exchange their relative performance across regimes, then varying the mixture weights \(\pi\) can reverse their global ranking while leaving all per-dataset results unchanged. Consequently, a global leaderboard is not an intrinsic property of a model, but rather a function of the benchmark composition.

This observation motivates a shift from global rankings to \emph{regime-conditioned evaluation}. CURVBENCH preserves two key structures: (i) curvature mismatch as a fundamental limitation on representation geometry, and (ii) regime-conditioned performance gaps as indicators of preference variation across geometric settings. All proofs in this section are in Appendix~\ref{app:theory-proofs}.

\subsection{Curvature mismatch as a metric obstruction}
\label{subsec:curv-distortion}

Let \((V,d_G)\) be the finite metric space induced by a graph \(G\), and let
\(h:V\to\mathcal M\) be an embedding into a representation metric space
\((\mathcal M,d_{\mathcal M})\). A quadruple \(q=(a,b,c;m)\) is valid if
\(b\ne c\), \(b,c\in\mathcal N(m)\), and \(0<d_G(a,m)<\infty\). For any metric
space \((\mathcal X,d_{\mathcal X})\) whose underlying points contain \(a,b,c,m\), we define
\begin{equation}
\label{eq:metric-curv-residual}
    \xi_{\mathcal X}(q)
    =
    \frac{
        d_{\mathcal X}(a,m)^2
        +\frac14 d_{\mathcal X}(b,c)^2
        -\frac12
        \left(
            d_{\mathcal X}(a,b)^2+d_{\mathcal X}(a,c)^2
        \right)
    }{
        2d_{\mathcal X}(a,m)
    },
\end{equation}
where $\xi_{\mathcal X}(q)$ is the midpoint curvature residual used by CURVBENCH, written for an arbitrary metric space. Since \(\xi_{\mathcal X}\) is homogeneous of degree one
in distances, we compare graph and representation metrics after a global scale
calibration \citep{bourgain1985lipschitz,linial1995geometry}. For \(\lambda>0\), define
\begin{equation}
\label{eq:scale-calibrated-distortion}
    \mathrm{dis}_{\infty}^{\lambda}(h;G,\mathcal M)
    =
    \max_{u,v:\,d_G(u,v)<\infty}
    \left|
        d_{\mathcal M}(h(u),h(v))-\lambda d_G(u,v)
    \right|.
\end{equation}
For any valid quadruple \(q=(a,b,c;m)\), we denote by
\(h(q)=(h(a),h(b),h(c);h(m))\) the corresponding quadruple after embedding.

\begin{theorem}[Curvature mismatch lower-bounds metric distortion]
\label{thm:curvature-stability}
Let \(D=\mathrm{diam}_f(G)\) and
\(\delta=\mathrm{dis}_{\infty}^{\lambda}(h;G,\mathcal M)\). If
\(\delta<\lambda/2\), then for every valid quadruple \(q\) in \(G\),
\begin{equation}
\label{eq:curv-stability-bound}
    \left|
        \xi_{\mathcal M}(h(q))-\lambda \xi_G(q)
    \right|
    \le
    C_D\delta,
\end{equation}
where one admissible constant is $
C_D
=
\frac12
+
\frac58(2D+1)^2
+
\frac54(2D+1).$

Consequently, for any distribution \(\mathcal Q_G\) over valid quadruples,
\begin{equation}
\label{eq:curv-exp-bound}
    \mathbb E_{q\sim\mathcal Q_G}
    \left[
        \left|
            \xi_{\mathcal M}(h(q))-\lambda \xi_G(q)
        \right|
    \right]
    \le
    C_D
    \,
    \mathrm{dis}_{\infty}^{\lambda}(h;G,\mathcal M).
\end{equation}
For an embedding class \(\mathcal H\), define $\mathrm{Dist}_{\mathcal H}^{\lambda}(G,\mathcal M)
    =
    \inf_{h\in\mathcal H}
    \mathrm{dis}_{\infty}^{\lambda}(h;G,\mathcal M),$ and


\begin{equation}
\label{eq:curv-mismatch-class}
    \eta_{\mathcal H}^{\lambda}(G,\mathcal M)
    =
    \inf_{h\in\mathcal H}
    \mathbb E_{q\sim\mathcal Q_G}
    \left[
        \left|
            \xi_{\mathcal M}(h(q))-\lambda \xi_G(q)
        \right|
    \right].
\end{equation}
If \(\mathrm{Dist}_{\mathcal H}^{\lambda}(G,\mathcal M)<\lambda/2\), then
\begin{equation}
\label{eq:dist-lower-bound}
    \mathrm{Dist}_{\mathcal H}^{\lambda}(G,\mathcal M)
    \ge
    \frac{
        \eta_{\mathcal H}^{\lambda}(G,\mathcal M)
    }{
        C_D
    }.
\end{equation}
\end{theorem}

The result follows from the fact that \(\xi\) is Lipschitz with respect to the four distances 
\(d(a,m)\), \(d(b,c)\), \(d(a,b)\), and \(d(a,c)\), provided the denominator is bounded away from zero; the condition \(\delta<\lambda/2\) ensures this for all valid quadruples.

This establishes a direct link between metric distortion and curvature preservation: any embedding that approximately preserves pairwise distances must also preserve midpoint curvature residuals up to controlled error. Conversely, systematic mismatch in curvature residuals certifies the presence of non-negligible metric distortion.

This provides a geometric explanation for curvature-dependent model behavior. Euclidean representations are naturally aligned with near-zero residuals, hyperbolic representations with negative residuals, and mixed or adaptive geometries with graphs whose curvature varies across regions.

\subsection{Regime-conditioned orders estimate preference variation}
\label{subsec:order-stability}

The previous result explains why different curvature regimes may favor different
representation geometries. We now show what CURVBENCH measures when it compares
model orders within and across regimes.

Let \(\mathcal M_0=\{M_1,\ldots,M_N\}\) be the evaluated models. For a graph
\(G\) and task \(t\), let \(\ell_i(G)\) be the empirical loss of model \(M_i\).
For tolerance \(\epsilon\ge0\), define the pairwise comparison state
\begin{equation}
\label{eq:pairwise-state}
    Z_{ij}^{\epsilon}(G)
    =
    \operatorname{sign}_{\epsilon}
    \bigl(\ell_j(G)-\ell_i(G)\bigr),
    \qquad
    \operatorname{sign}_{\epsilon}(x)
    =
    \operatorname{sign}(x)\mathbf 1\{|x|\ge\epsilon\}.
\end{equation}
The state \(0\) treats statistically indistinguishable models as incomparable,
so the induced order is partial rather than artificially total. For each regime
\(r\in\mathcal R\), define
\begin{equation}
\label{eq:state-prob}
    p_{ij,r}^{z}
    =
    \mathbb P_{G\sim\mathcal D_{r,t}}
    \left(
        Z_{ij}^{\epsilon}(G)=z
    \right),
    \qquad
    z\in\{-1,0,+1\}.
\end{equation}
The normalized distance between two empirical partial orders is
\begin{equation}
\label{eq:partial-kendall}
    d_{\epsilon}(G,G')
    =
    \frac{1}{\binom{N}{2}}
    \sum_{i<j}
    \mathbf 1
    \left\{
        Z_{ij}^{\epsilon}(G)
        \ne
        Z_{ij}^{\epsilon}(G')
    \right\}.
\end{equation}

\begin{theorem}[Within--cross order gap is between-regime preference variance]
\label{thm:partial-order-variance}
Assume \(K=|\mathcal R|\ge2\) regimes are sampled uniformly. Let \(G,G'\) be independent graphs.  Under within-regime sampling, both graphs are drawn from the
same randomly selected regime. Under cross-regime sampling, they are drawn from
two different randomly selected regimes. Then
\begin{align}
\label{eq:partial-order-gap}
&
\mathbb E
\left[
    d_{\epsilon}(G,G')
    \mid
    \mathrm{cross}
\right]
-
\mathbb E
\left[
    d_{\epsilon}(G,G')
    \mid
    \mathrm{within}
\right]
\nonumber\\
&\qquad
=
\frac{K}{
    (K-1)\binom{N}{2}
}
\sum_{i<j}
\sum_{z\in\{-1,0,+1\}}
\mathrm{Var}_{r\sim\mathrm{Unif}(\mathcal R)}
\left(
    p_{ij,r}^{z}
\right)
\ge 0.
\end{align}
The gap is zero if and only if every model-pair comparison distribution
\(\{p_{ij,r}^{z}\}_{z}\) is invariant across curvature regimes.
\end{theorem}

For a fixed model pair \((i,j)\), within-regime agreement depends on 
\(\sum_z (p_{ij,r}^z)^2\), while cross-regime depends on 
\(\sum_z p_{ij,r}^z p_{ij,s}^z\) for \(r \ne s\). The difference between these quantities is exactly the variance of the comparison-state probabilities across regimes. Averaging over all model pairs yields Eq.~\eqref{eq:partial-order-gap}.

Therefore, the within--cross order gap is not merely an empirical observation; it provides a direct estimate of between-regime variation in model preferences. If curvature has no effect on model behavior, then \(p_{ij,r}^z\) is invariant across regimes and the gap vanishes. Conversely, if different curvature regimes favor different inductive biases, the variance is strictly positive, and cross-regime rankings are necessarily less stable.

A global leaderboard collapses this variability into a single aggregate score, obscuring regime-dependent behavior. In contrast, regime-conditioned partial orders preserve this information, enabling a more faithful comparison of models across heterogeneous graph geometries.

\section{Experimental Results and Analyses}

\subsection{Curvature-stratified Model Rankings}

Rather than relying on a single global leaderboard, we evaluate whether model-induced rankings are more consistent within the same curvature regime than across different regimes. This aligns with the central premise of CURVBENCH: model performance is governed by the interaction between a model’s geometric inductive bias and the intrinsic geometry of the graph.

\begin{table}[htbp]
\centering
\caption{Performance on Node Classification task.  We highlight the top-3 results with \textbf{\textcolor{red!80!black}{red bolded}}, \textcolor{red}{red} and \textbf{bolded}.}
\resizebox{\linewidth}{!}{
\begin{tabular}{c|ccc|ccc|ccc}
\toprule
Baselines & Cora & Citeseer & PubMed & Airport & Cornell & Actor & Disease & Telecom & CS\_Phds \\
\midrule
GCN       & 80.36$\pm$0.71   & 68.68$\pm$0.65   & 78.12$\pm$0.28   & 79.18$\pm$0.98   & 38.37$\pm$3.52   & 31.31$\pm$0.62   & 83.82$\pm$5.58   & 85.85$\pm$0.64   & 35.51$\pm$2.87   \\
GAT       & 80.72$\pm$0.70   & 67.50$\pm$1.64   & 77.08$\pm$0.32   & 82.82$\pm$0.78   & 44.32$\pm$4.52   & 28.67$\pm$0.60   & \textbf{90.62$\pm$1.41}  & 79.73$\pm$0.19   & 26.83$\pm$0.00   \\
GraphSAGE & \cellcolor{yellow!30}\textbf{\textcolor{red!80!black}{88.30$\pm$0.21}}  & \cellcolor{yellow!30}\textcolor{red}{74.89$\pm$0.65}   & \cellcolor{yellow!30}\textcolor{red}{88.48$\pm$0.05}   & 48.80$\pm$0.27   & \textbf{\textcolor{red!80!black}{73.51$\pm$3.52}}   & 32.84$\pm$0.56   & \textcolor{red}{95.60$\pm$1.45}   & 92.90$\pm$3.08   & 26.73$\pm$6.15   \\
MLP       & 56.12$\pm$1.05   & 54.18$\pm$0.87   & 71.27$\pm$0.38   & \cellcolor{yellow!30}\textcolor{red}{85.07$\pm$0.55}   & \cellcolor{yellow!30}\textcolor{red}{68.10$\pm$2.26}   & \cellcolor{yellow!30}\textbf{\textcolor{red!80!black}{37.46$\pm$0.62}}   & 79.90$\pm$0.00   & 88.15$\pm$0.04   & 26.83$\pm$0.00   \\
PCNet     & \cellcolor{yellow!30}\textcolor{red}{88.08$\pm$0.44}   & \cellcolor{yellow!30}\textbf{\textcolor{red!80!black}{75.59$\pm$0.25}}   & \cellcolor{yellow!30}\textbf{\textcolor{red!80!black}{89.97$\pm$0.11}}   & \cellcolor{yellow!30}45.51$\pm$0.13   & \cellcolor{yellow!30}\textbf{61.08$\pm$4.52}   & \cellcolor{yellow!30}\textbf{33.45$\pm$0.97}   & 78.56$\pm$0.92   & 87.49$\pm$0.04   & 31.51$\pm$0.44   \\
\midrule
HAT       & \cellcolor{yellow!30}\textbf{81.60$\pm$0.32}   & \cellcolor{yellow!30}\textbf{70.99$\pm$0.28}   & \cellcolor{yellow!30}\textbf{78.74$\pm$0.46}   & 59.22$\pm$5.53   & 36.84$\pm$0.03   & \textcolor{red}{34.64$\pm$0.44}   & 77.51$\pm$0.30   & 87.92$\pm$0.02   & 26.82$\pm$0.00   \\
HGNN      & 78.52$\pm$0.63   & 67.62$\pm$0.81   & 76.54$\pm$0.43   & \cellcolor{yellow!30}\textbf{83.51$\pm$2.47}   & \cellcolor{yellow!30}\textbf{61.08$\pm$1.32}   & \cellcolor{yellow!30}28.92$\pm$0.68   & 77.72$\pm$2.15   & \textbf{93.16$\pm$0.97}   & 24.41$\pm$2.87   \\
HyboNet & 75.16$\pm$0.84   & 70.23$\pm$1.20   & 73.58$\pm$0.45   & 60.88$\pm$4.17   & 36.22$\pm$1.06   & 26.67$\pm$1.32   & 77.01$\pm$4.59   & 62.03$\pm$7.32   & 26.73$\pm$0.19   \\
HGCNN     & 76.74$\pm$0.78   & 67.22$\pm$1.01   & 75.88$\pm$0.33   & 60.23$\pm$2.20   & \textbf{61.08$\pm$0.96}   & 28.80$\pm$0.23   & \cellcolor{yellow!30}77.92$\pm$1.56   & \cellcolor{yellow!30}\textbf{93.16$\pm$1.70}   & \cellcolor{yellow!30}\textcolor{red}{43.63$\pm$2.86}   \\
\midrule
CUSP      & 76.94$\pm$0.95   & 68.20$\pm$1.28   & 66.36$\pm$2.31   & 58.65$\pm$2.24   & 40.54$\pm$1.00   & 24.81$\pm$1.26   & 85.79$\pm$1.87   & 66.73$\pm$5.01   & 29.65$\pm$3.47   \\
QGCN      & 79.80$\pm$0.41   & 67.32$\pm$0.26   & 75.90$\pm$1.03   & 61.07$\pm$0.74   & 54.59$\pm$2.02   & 26.74$\pm$0.55   & \cellcolor{yellow!30}83.31$\pm$1.42   & \cellcolor{yellow!30}\textbf{\textcolor{red!80!black}{98.25$\pm$0.05}}   & \cellcolor{yellow!30}\textbf{\textcolor{red!80!black}{45.39$\pm$2.33}}   \\
\midrule
GraphMoRE & 81.06$\pm$0.33   & 68.30$\pm$0.78   & 76.34$\pm$1.12   & \textbf{\textcolor{red!80!black}{90.42$\pm$1.32}}   & 40.54$\pm$3.42   & 24.49$\pm$0.81   & \cellcolor{yellow!30}\textbf{\textcolor{red!80!black}{96.11$\pm$0.77}}   & \cellcolor{yellow!30}\textcolor{red}{93.40$\pm$0.31}   & \cellcolor{yellow!30}\textbf{37.45$\pm$2.82}   \\
\bottomrule
\end{tabular}
}
\label{tab:app-nc-acc-results}
\end{table}

\begin{table}[htbp]
\centering
\caption{AUC results on Link Prediction (LP) task. Datasets and baselines are divided into different regimes. We highlight the top-3 results with \textbf{\textcolor{red!80!black}{red bolded}}, \textcolor{red}{red} and \textbf{bold}.}
\label{tab:app-lp-auc-results}
\resizebox{\linewidth}{!}{
\begin{tabular}{c|ccc|ccc|ccc}
\toprule
Baselines & Cora & Citeseer & PubMed & Airport & Cornell & Actor & Disease & Telecom & CS\_Phds \\
\midrule
GCN       & \cellcolor{yellow!30}\textbf{91.87$\pm$1.01}   & \cellcolor{yellow!30}92.09$\pm$0.87   & \cellcolor{yellow!30}92.82$\pm$0.31   & 94.79$\pm$0.57   & 64.60$\pm$5.81   & 80.11$\pm$0.65   & 47.78$\pm$4.97   & 71.53$\pm$0.40   & 42.50$\pm$0.98   \\
GAT       & \cellcolor{yellow!30}\textcolor{red}{92.07$\pm$0.42}   & \cellcolor{yellow!30}\textbf{92.83$\pm$0.62}   & \cellcolor{yellow!30}91.77$\pm$0.20   & \cellcolor{yellow!30}93.80$\pm$0.45   & \cellcolor{yellow!30}\textbf{\textcolor{red!80!black}{70.06$\pm$8.74}}   & \cellcolor{yellow!30}\textbf{80.48$\pm$0.88}   & 49.64$\pm$5.47   & \textcolor{red}{73.87$\pm$0.61}   & \textbf{61.35$\pm$4.24}   \\
GraphSAGE & 72.74$\pm$2.42   & 74.40$\pm$1.73   & 84.02$\pm$1.16   & 79.04$\pm$0.83   & 58.59$\pm$9.80   & 61.34$\pm$1.21   & 49.95$\pm$0.11   & 60.52$\pm$1.05   & 36.91$\pm$0.02   \\
MLP       & 83.08$\pm$0.96   & 88.88$\pm$1.54   & 86.97$\pm$0.56   & 91.67$\pm$0.62   & 66.06$\pm$6.03   & 69.95$\pm$0.67   & 50.72$\pm$4.98   & \textbf{73.76$\pm$0.76}   & 50.00$\pm$0.00   \\
PCNet     & 74.40$\pm$1.94   & 72.24$\pm$1.38   & 93.39$\pm$0.27   & 79.67$\pm$0.78   & 54.61$\pm$9.92   & 61.69$\pm$0.68   & \cellcolor{yellow!30}59.11$\pm$3.97   & \cellcolor{yellow!30}64.61$\pm$1.02   & \cellcolor{yellow!30}\textbf{\textcolor{red!80!black}{84.28$\pm$1.22}}   \\
\midrule
HGNN      & 74.33$\pm$0.49   & 86.29$\pm$0.38   & 92.23$\pm$0.20   & 92.81$\pm$0.27   & 66.02$\pm$5.47   & 68.79$\pm$0.32   & 51.87$\pm$1.77   & \textbf{\textcolor{red!80!black}{76.23$\pm$0.20}}   & 49.72$\pm$4.45   \\
HGCNN     & 83.54$\pm$1.72   & 88.03$\pm$0.66   & \textcolor{red}{93.83$\pm$0.06}   & 93.33$\pm$0.56   & \textcolor{red}{69.75$\pm$0.58}   & 79.37$\pm$0.35   & \textbf{62.59$\pm$7.89}   & 56.01$\pm$1.70   & 52.45$\pm$1.46   \\
HyboNet   & 87.88$\pm$2.11   & 77.40$\pm$0.82   & 91.19$\pm$0.60   & \cellcolor{yellow!30}\textcolor{red}{96.29$\pm$0.63}   & \cellcolor{yellow!30}\textbf{66.99$\pm$3.77}   & \cellcolor{yellow!30}\textcolor{red}{82.18$\pm$1.05}   & 40.56$\pm$2.89   & 49.84$\pm$0.61   & \textcolor{red}{62.54$\pm$4.50}   \\
\midrule
CUSP      & 88.33$\pm$1.56   & \textcolor{red}{93.98$\pm$1.40}   & 63.74$\pm$0.70   & 75.74$\pm$0.98   & 64.67$\pm$7.20   & 71.93$\pm$1.11   & 27.20$\pm$5.78   & 68.39$\pm$0.75   & 40.34$\pm$5.84   \\
QGCN      & 88.00$\pm$0.42   & 87.24$\pm$0.59   & \textbf{93.51$\pm$0.20}   & \textbf{95.92$\pm$0.14}   & 62.58$\pm$3.66   & 79.19$\pm$0.41   & \cellcolor{yellow!30}\textcolor{red}{72.74$\pm$1.53}   & \cellcolor{yellow!30}65.64$\pm$0.67   & \cellcolor{yellow!30}54.29$\pm$1.50   \\
\midrule
GraphMoRE & \cellcolor{yellow!30}\textbf{\textcolor{red!80!black}{96.50$\pm$0.36}}   & \cellcolor{yellow!30}\textbf{\textcolor{red!80!black}{98.39$\pm$0.18}}   & \cellcolor{yellow!30}\textbf{\textcolor{red!80!black}{98.63$\pm$0.20}}   & \cellcolor{yellow!30}\textbf{\textcolor{red!80!black}{96.30$\pm$0.42}}   & \cellcolor{yellow!30}64.19$\pm$4.81   & \cellcolor{yellow!30}\textbf{\textcolor{red!80!black}{85.54$\pm$0.45}}   & \cellcolor{yellow!30}\textbf{\textcolor{red!80!black}{74.22$\pm$2.30}}   & \cellcolor{yellow!30}72.92$\pm$0.44   & \cellcolor{yellow!30}42.10$\pm$7.47   \\
\bottomrule
\end{tabular}
}
\end{table}

\paragraph{Observation 1: Curvature regimes induce coherent top-model partial orders.}

We first test whether curvature regimes organize model behavior by comparing dataset-induced model rankings. For each dataset, models are sorted by mean performance and converted into a top-3 truncated ranking: the top three models retain their exact ranks, while all remaining models are grouped into a shared lower tier. This focuses the analysis on the dominant inductive biases revealed by the benchmark.

On node classification, the proposed curvature grouping yields substantially higher within-regime than cross-regime ranking consistency. The average top-3 truncated Spearman correlation~\citep{spearman1904proof} is $0.539$ within regimes versus $0.036$ across regimes, corresponding to a gap of $0.503$. The same trend holds for Kendall correlation ($0.512$ vs.\ $0.031$)~\citep{kendall1938new} and top-3 Jaccard overlap ($0.567$ vs.\ $0.159$)~\citep{jaccard1901etude}.

Importantly, among all $280$ balanced $3$-$3$-$3$ partitions of the nine datasets, the proposed curvature grouping achieves the \textbf{largest within-minus-cross gap} under all three metrics, yielding an exact regrouping significance of $p = 1/280 = 0.0036$.

Results in Table~\ref{tab:app-nc-acc-results} and Table~\ref{tab:app-lp-auc-results} provide strong evidence that curvature is not merely a post-hoc descriptor, but a meaningful organizing principle for model behavior. Datasets within the same curvature regime exhibit substantially more similar top-model preferences than datasets across regimes, supporting the use of curvature-conditioned partial orders rather than a single global leaderboard.
\paragraph{Observation 2: Curvature regimes reorganize effective inductive biases.}

We next examine which model families dominate top-performing positions across regimes. Models are grouped by their geometric inductive biases: Euclidean, fixed hyperbolic, mixed or pseudo-Riemannian, and adaptive Riemannian. The composition of top-performing models changes markedly with curvature.

In the near-zero regime, Euclidean methods dominate, winning all three datasets and occupying \(66.7\%\) of the top-3 positions. Moreover, Cora, Citeseer, and PubMed share an almost identical top-3 set—GraphSAGE, PCNet, and HAT—up to minor rank variations, indicating that flat aggregation and spectral filtering are well aligned with near-zero curvature structure.

In the positive regime, Euclidean models remain competitive, winning two of three datasets and accounting for \(55.6\%\) of top-3 placements. However, the inductive bias shifts: MLP appears in the top-3 across all datasets, suggesting that in positively curved (compact or clustered) graphs, node attributes become more predictive than relational structure.

The negative regime exhibits the most pronounced reorganization. Mixed and adaptive Riemannian methods win all three datasets and occupy \(55.6\%\) of top-3 positions, while Euclidean methods drop to \(0\%\). This sharp transition indicates that Euclidean models are not inherently weak, but rather regime-limited: their failure in negatively curved graphs reflects a mismatch between flat geometry and hierarchical or tree-like structure.

We further quantify this interaction via variance decomposition after per-dataset normalization. Incorporating a family-by-regime interaction term increases explained variance from \(0.0596\) to \(0.1063\), with the interaction accounting for \(43.95\%\) of the explained variance. Results show that model performance is governed by the alignment between geometric inductive bias and graph curvature, rather than by a globally dominant model family.

\subsection{Geometry-conditioned behavior of Graph Foundation Models}

\begin{table}[h]
\centering
\caption{Performance of GFMs under 1-shot and 5-shot scenarios. OOM means Out-Of-Memory.}
\resizebox{\linewidth}{!}{
\begin{tabular}{c|ccc|ccc|ccc}
\toprule
\multicolumn{10}{c}{\textbf{1-shot scenario}} \\
\midrule
Baselines   & Cora          & Citeseer      & PubMed        & Airport       & Cornell       & Actor        & Disease       & Telecom       & CS$\_$Phds      \\
\midrule
GCOPE       & 33.19$\pm$6.05  & \textbf{37.38$\pm$7.46}  & 41.49$\pm$4.35  & \textbf{19.22$\pm$8.35}  & 24.62$\pm$9.36  & \textbf{\textcolor{red!80!black}{24.30$\pm$1.85}} & \textcolor{red}{73.08$\pm$12.69} & \textbf{\textcolor{red!80!black}{54.82$\pm$13.10}} & \textbf{\textcolor{red!80!black}{26.21$\pm$2.11}} \\
MDGPT       & \textcolor{red}{44.58$\pm$7.83}  & \textcolor{red}{39.04$\pm$10.53} & \textbf{\textcolor{red!80!black}{53.36$\pm$10.72}} & 18.28$\pm$17.07 & 29.26$\pm$6.27  & 20.01$\pm$4.33 & 52.42$\pm$9.43  & \textbf{36.56$\pm$12.55} & 25.29$\pm$2.30 \\
MDGFM       & \textbf{43.27$\pm$7.28}  & \textbf{\textcolor{red!80!black}{41.20$\pm$6.31}}  & \textcolor{red}{51.52$\pm$9.34}  & 18.70$\pm$5.03  & \textbf{\textcolor{red!80!black}{35.14$\pm$9.02}}  & \textbf{20.74$\pm$2.15} & 57.84$\pm$10.77 & OOM           & 25.56$\pm$2.11 \\
SAMGPT      & \textbf{\textcolor{red!80!black}{44.64$\pm$14.94}} & 36.03$\pm$8.41  & 45.24$\pm$8.45  & 19.12$\pm$9.20   & \textcolor{red}{33.84$\pm$8.54}  & 19.72$\pm$5.88 & \textbf{60.28$\pm$11.04} & \textcolor{red}{45.12$\pm$13.49} & 25.36$\pm$6.92 \\
GraphGluing & 32.22$\pm$1.33  & 28.48$\pm$6.59  & \textbf{45.90$\pm$4.70} & \textbf{\textcolor{red!80!black}{41.37$\pm$2.77}}  & \textbf{32.51$\pm$11.25} & \textcolor{red}{24.10$\pm$2.25} & \textbf{\textcolor{red!80!black}{79.67$\pm$0.18}}  & OOM           & \textcolor{red}{26.15$\pm$2.45} \\
SA2GFM      & 40.25$\pm$8.05  & 29.98$\pm$7.81  & 45.79$\pm$8.90  & \textcolor{red}{25.63$\pm$5.95}  & 20.99$\pm$5.45  & 18.53$\pm$2.09 & 51.12$\pm$13.39 & OOM           & \textbf{25.92$\pm$2.75} \\
\midrule
\multicolumn{10}{c}{\textbf{5-shot scenario}} \\
\midrule
Baselines   & Cora         & Citeseer     & PubMed       & Airport       & Cornell       & Actor        & Disease       & Telecom       & CS$\_$Phds      \\
\midrule
GCOPE       & 61.40$\pm$1.88 & 52.42$\pm$5.26 & 58.56$\pm$1.79 & 20.95$\pm$4.32  & \textbf{\textcolor{red!80!black}{68.03$\pm$4.33}}  & \textbf{\textcolor{red!80!black}{24.55$\pm$2.07}} & \textcolor{red}{79.44$\pm$0.58}  & \textbf{\textcolor{red!80!black}{72.16$\pm$8.37}}  & 26.70$\pm$1.95 \\
MDGPT       & \textbf{60.86$\pm$4.86} & \textbf{\textcolor{red!80!black}{58.68$\pm$6.93}} & \textbf{59.86$\pm$6.83} & \textbf{22.78$\pm$10.44} & 44.98$\pm$7.18  & \textbf{21.28$\pm$4.21} & 54.68$\pm$9.71  & \textbf{38.74$\pm$9.13}  & \textcolor{red}{26.86$\pm$2.27} \\
MDGFM       & \textbf{\textcolor{red!80!black}{64.93$\pm$4.43}} & \textcolor{red}{58.10$\pm$4.55} & \textcolor{red}{65.65$\pm$5.30} & 19.92$\pm$3.88  & \textcolor{red}{60.10$\pm$7.78}  & 21.12$\pm$1.67 & 63.55$\pm$8.69  & OOM           & \textbf{26.81$\pm$2.42} \\
SAMGPT      & \textcolor{red}{64.62$\pm$9.89} & \textbf{53.76$\pm$5.70} & 56.16$\pm$7.27 & 21.28$\pm$6.74  & \textbf{52.24$\pm$6.18}  & 19.92$\pm$6.24 & \textbf{68.32$\pm$9.88}  & \textcolor{red}{58.56$\pm$11.62} & \textbf{\textcolor{red!80!black}{27.12$\pm$6.40}} \\
GraphGluing & 52.52$\pm$6.06 & 44.05$\pm$2.08 & \textbf{\textcolor{red!80!black}{66.14$\pm$1.71}} & \textbf{\textcolor{red!80!black}{42.46$\pm$1.11}}  & 40.33$\pm$10.72 & \textcolor{red}{23.47$\pm$1.75} & \textbf{\textcolor{red!80!black}{80.42$\pm$0.73}}  & OOM           & 26.63$\pm$1.67 \\
SA2GFM      & 50.91$\pm$6.57 & 38.25$\pm$4.18 & 53.40$\pm$8.94 & \textcolor{red}{25.95$\pm$9.77}  & 22.83$\pm$7.34  & 19.35$\pm$2.96 & 56.77$\pm$10.84 & OOM           & 26.05$\pm$1.87 \\
\bottomrule
\end{tabular}
}
\label{GFM results}
\end{table}

We further evaluate recent GFMs under the same curvature-stratified protocol. The evaluated methods span representative transfer mechanisms, including virtual-node coordination, domain/structure tokenization, topology alignment, Riemannian gluing, and structure-aware augmentation. This setup enables a sharper question: do GFMs eliminate the need for curvature-aware evaluation, or merely introduce new geometry-conditioned inductive biases?

\paragraph{Observation 3: 1-shot GFM rankings remain curvature-conditioned.}

The 1-shot regime is most diagnostic, as performance is dominated by transferred inductive bias rather than target supervision. Under top-3 truncated rankings, model orderings are substantially more consistent within curvature regimes than across them: Spearman correlation increases from $-0.102$ to $0.269$, and Kendall correlation from $-0.086$ to $0.222$. Among all $280$ balanced $3$-$3$-$3$ partitions of the nine datasets, our curvature grouping ranks in the top $5\%$ (Spearman) and $6.1\%$ (Kendall), confirming that it captures a meaningful axis of transfer behavior.

Crucially, the leading GFM varies by regime. On near-zero graphs, MDGPT and MDGFM achieve the best 1-shot means ($45.66$ and $45.33$), indicating that token- or prompt-based alignment is well-suited to citation-like structures. On positive graphs, GraphGluing becomes dominant (mean $32.66$), largely driven by Airport. On negative graphs, GraphGluing again achieves the strongest available-case performance, while GCOPE is the best complete-coverage method due to OOM failures of several competitors on Telecom.

Overall, GFMs do not collapse CURVBENCH into a geometry-agnostic leaderboard. Instead, performance depends on how inductive biases interact with curvature. A flat 1-shot average favors GraphGluing if OOM cases are ignored, whereas a complete-coverage view favors GCOPE. The regime-aware view is more informative: MDGPT/MDGFM excel on near-zero graphs, GraphGluing dominates structurally demanding regimes when feasible, and GCOPE offers the most robust coverage.

\paragraph{Observation 4: Extra labels reveal a geometry--scalability frontier.}

Increasing supervision from 1-shot to 5-shot yields highly uneven gains across regimes. Averaged over GFMs, near-zero graphs improve by $15.95$, positive graphs by $6.97$, and negative graphs by only $4.49$. Dataset-level trends reinforce this pattern: Airport gains $1.84$, Actor $0.38$, and CS$\_$Phds $0.95$. This suggests that citation-like graphs are primarily label-limited, whereas several non-citation graphs remain geometry- or structure-limited even with additional supervision.

This exposes a fundamental trade-off. Geometry-intensive methods such as GraphGluing can achieve strong performance on structurally distinctive regimes, but their advantage is conditional on scalability. In contrast, lighter coordination or prompting approaches—especially GCOPE—are less specialized but provide broader coverage and stronger label elasticity (e.g., GCOPE improves by $14.43$ and is among the few methods that handle Telecom).

Therefore, GFM evaluation should jointly report curvature regime, shot elasticity, and feasibility, as OOM behavior is itself a manifestation of transfer difficulty rather than an implementation detail.

\subsection{Performance on table-derived graphs}

\begin{table}[h]
\centering
\caption{Performance on table-derived graphs.}
\label{tab:table-derived-with-fig}
\resizebox{0.65\linewidth}{!}{
\begin{tabular}{c|ccccc}
\toprule
Baselines & Carcinogenesis & Hepatitis    & PTE          & Toxicology   & F1          \\
\midrule
GCN       & 57.27$\pm$5.07     & \textcolor{red}{83.19$\pm$0.44}   & 79.66$\pm$1.82   & \textbf{54.78$\pm$1.58}   & \textbf{4.70$\pm$0.70}   \\
GAT       & 60.30$\pm$4.59     & 79.80$\pm$1.29   & 78.33$\pm$3.11   & 52.75$\pm$1.29   & 4.25$\pm$0.14   \\
GraphSAGE & \textcolor{red}{65.45$\pm$1.27}     & \textbf{81.80$\pm$1.30}   & \textcolor{red}{81.67$\pm$0.00}   & \textbf{\textcolor{red!80!black}{55.07$\pm$1.02}}   & 4.10$\pm$0.14   \\
MLP       & 54.55$\pm$0.00     & 70.80$\pm$1.78   & 79.00$\pm$0.91   & \textbf{\textcolor{red!80!black}{55.07$\pm$0.20}}   & 3.96$\pm$0.40   \\
PCNet     & 53.03$\pm$2.14     & \textbf{\textcolor{red!80!black}{84.20$\pm$1.92}}   & \textbf{81.00$\pm$1.49}   & 52.46$\pm$0.65   & 3.90$\pm$0.27   \\
\midrule
HGNN      & 62.42$\pm$2.42   & 66.80$\pm$0.40 & 77.00$\pm$1.25 & 53.33$\pm$2.13 & 4.02$\pm$0.46 \\
HAT       & \textbf{\textcolor{red!80!black}{70.84$\pm$1.47}}     & 59.19$\pm$0.44   & \textbf{\textcolor{red!80!black}{85.66$\pm$3.02}}   & 40.57$\pm$4.09   & \textbf{\textcolor{red!80!black}{40.84$\pm$5.77}}  \\
HGCNN     & 61.21$\pm$1.21   & 64.20$\pm$0.40 & 65.33$\pm$3.86 & 51.59$\pm$1.48 & \textbf{4.73$\pm$0.16}  \\
HyboNet & 43.63$\pm$1.76     & 67.59$\pm$5.38   & 43.33$\pm$5.55   & 44.92$\pm$5.55   & 4.11$\pm$0.21   \\
\midrule
CUSP      & 57.57$\pm$7.66     & 80.40$\pm$1.20   & 51.66$\pm$10.90  & \textcolor{red}{54.87$\pm$0.57}   & \textcolor{red}{5.04$\pm$0.25}   \\
QGCN      & \textbf{63.33$\pm$2.42}   & 67.20$\pm$2.23 & 55.33$\pm$1.25 & 53.04$\pm$2.35 & 4.47$\pm$0.50 \\
\midrule
GraphMoRE & 54.55$\pm$5.07   & 81.00$\pm$1.67 & 78.33$\pm$1.05  & 53.91$\pm$0.58 & 4.16$\pm$0.12 \\
\bottomrule
\end{tabular}
}
\end{table}

\paragraph{Observation 5: Table-derived graphs expose tail-driven geometry.}
As shown in Table~\ref{tab:table-derived-with-fig}, model behavior on table-derived graphs is highly non-uniform. HAT wins three of five datasets and dominates F1 (40.84 vs. 5.04), but drops to the bottom tier on Hepatitis and Toxicology, where Euclidean methods remain strong. This suggests a construction-induced mixed regime: performance depends less on mean curvature 
\(\bar{\kappa}(G)\) and more on asymmetric (tail) curvature mass introduced by table construction. Hyperbolic attention acts as a high-variance specialist, excelling when curvature tails are pronounced, while Euclidean methods provide more stable baselines.

\begin{wrapfigure}{r}{0.6\linewidth}
    \centering
    \includegraphics[width=\linewidth]{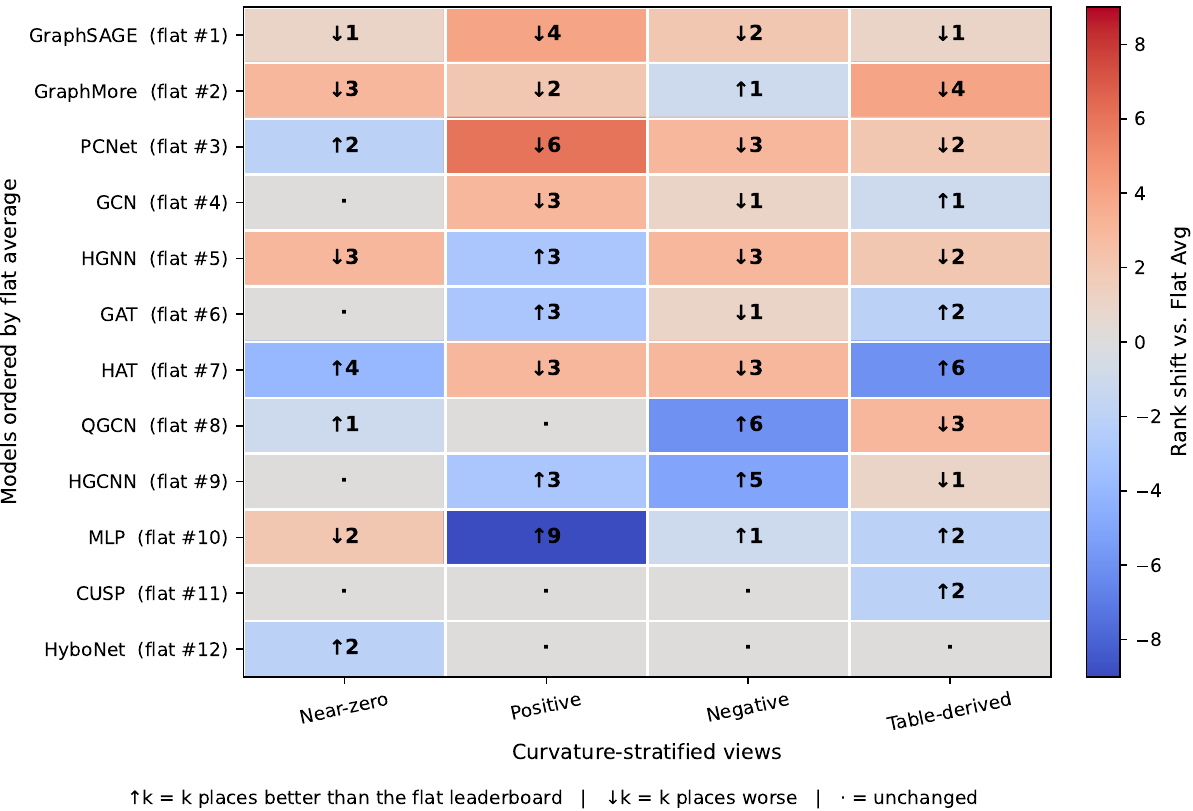}
    \caption{Rank-shift heatmap.}
    \label{fig:table-derived-diagnostic}
\end{wrapfigure}

Thus, both \(\bar{\kappa}(G)\) and \(\gamma_{\kappa}(G)\) are needed to explain performance; a flat average obscures this specialist–robustness trade-off.iled curvature profiles, where a small subset of nodes or edges carries disproportionately large geometric signal.

To sum up, Figure~\ref{fig:table-derived-diagnostic} demonstrates that model rankings are highly sensitive to the underlying data regime. A model's standing can shift substantially once evaluation is conditioned on curvature-stratified datasets, indicating that model performance is not an intrinsic, context-independent property. 
Additional analyses in Appendix~\ref{app:additional-analysis} reinforce these conclusions through several diagnostic lenses: We find that flat leaderboards induce significant task-dependent rank distortions, obscuring localized performance gains; While broad within-task trends remain stable, specific metric choices can alter the composition of "top-tier" model sets, particularly in non-flat regimes; Node Classification and Link Prediction emphasize distinct structural signals, even when evaluated within the same curvature regime; Evaluation of GFMs must transcend few-shot accuracy to include label elasticity and coverage. 
Together, these results demonstrate that a reliable evaluation of relational learning must jointly account for geometry, task objectives, metric selection, and computational feasibility. Reducing model capability to a single aggregate score is no longer sufficient for the current research landscape.

\section{Conclusion and Future Work}

We introduced \textsc{CurvBench}, a comprehensive curvature-stratified benchmark for relational learning. Our evaluation demonstrates that conventional flat leaderboards often obscure significant regime-dependent behaviors; while model rankings remain coherent within specific curvature regimes, they shift substantially across different geometries, tasks, and metrics. Consequently, model efficacy should be interpreted as the alignment between a model’s inductive bias and the structural geometry of the data, rather than a singular, context-independent score. Reliable assessment must therefore transcend simple accuracy to incorporate label elasticity, coverage, specialist–robustness trade-offs, and computational efficiency. These findings suggest two parallel trajectories for the field. First, evaluation protocols must move toward stratified, coverage-aware diagnostics that reveal the specific conditions under which a model succeeds. Second, future relational learning systems should transition from fixed inductive biases toward geometry-adaptive architectures capable of inferring, selecting, or composing appropriate biases directly from the data. By making benchmark design structurally aware, we can better characterize the fundamental strengths and limitations of the next generation of graph and foundation models.

\newpage
\bibliographystyle{unsrt}
\normalem
\bibliography{example_paper.bib}

@inproceedings{kipf2017semi,
  title={Semi-Supervised Classification with Graph Convolutional Networks},
  author={Kipf, Thomas N. and Welling, Max},
  booktitle={International Conference on Learning Representations},
  year={2017},
  url={https://openreview.net/forum?id=SJU4ayYgl}
}

@inproceedings{velickovic2018graph,
  title={Graph Attention Networks},
  author={Veli{\v{c}}kovi{\'c}, Petar and Cucurull, Guillem and Casanova, Arantxa and Romero, Adriana and Li{\`o}, Pietro and Bengio, Yoshua},
  booktitle={International Conference on Learning Representations},
  year={2018},
  url={https://openreview.net/forum?id=rJXMpikCZ}
}

@inproceedings{hamilton2017inductive,
  title={Inductive Representation Learning on Large Graphs},
  author={Hamilton, William L. and Ying, Rex and Leskovec, Jure},
  booktitle={Advances in Neural Information Processing Systems},
  year={2017},
  url={https://proceedings.neurips.cc/paper/2017/hash/5dd9db5e033da9c6fb5ba83c7a7ebea9-Abstract.html}
}

@inproceedings{hu2020open,
  title={Open Graph Benchmark: Datasets for Machine Learning on Graphs},
  author={Hu, Weihua and Fey, Matthias and Zitnik, Marinka and Dong, Yuxiao and Ren, Hongyu and Liu, Bowen and Catasta, Michele and Leskovec, Jure},
  booktitle={Advances in Neural Information Processing Systems},
  year={2020},
  url={https://proceedings.neurips.cc/paper/2020/hash/fb60d411a5c5b72b2e7d3527cfc84fd0-Abstract.html}
}

@article{dwivedi2023benchmarking,
  title={Benchmarking Graph Neural Networks},
  author={Dwivedi, Vijay Prakash and Joshi, Chaitanya K. and Luu, Anh Tuan and Laurent, Thomas and Bengio, Yoshua and Bresson, Xavier},
  journal={Journal of Machine Learning Research},
  volume={24},
  number={43},
  pages={1--48},
  year={2023},
  url={https://jmlr.org/papers/v24/22-0567.html}
}

@inproceedings{you2020design,
  title={Design Space for Graph Neural Networks},
  author={You, Jiaxuan and Ying, Rex and Leskovec, Jure},
  booktitle={Advances in Neural Information Processing Systems},
  year={2020},
  url={https://proceedings.neurips.cc/paper/2020/hash/c5c3d4fe6b2cc463c7d7ecba17cc9de7-Abstract.html}
}

@inproceedings{nickel2017poincare,
  title={Poincar{\'e} Embeddings for Learning Hierarchical Representations},
  author={Nickel, Maximilian and Kiela, Douwe},
  booktitle={Advances in Neural Information Processing Systems},
  year={2017},
  url={https://proceedings.neurips.cc/paper/2017/hash/59dfa2df42d9e3d41f5b02bfc32229dd-Abstract.html}
}

@inproceedings{ganea2018hyperbolic,
  title={Hyperbolic Neural Networks},
  author={Ganea, Octavian-Eugen and B{\'e}cigneul, Gary and Hofmann, Thomas},
  booktitle={Advances in Neural Information Processing Systems},
  year={2018},
  url={https://arxiv.org/abs/1805.09112}
}

@inproceedings{liu2019hyperbolic,
  title={Hyperbolic Graph Neural Networks},
  author={Liu, Qi and Nickel, Maximilian and Kiela, Douwe},
  booktitle={Advances in Neural Information Processing Systems},
  year={2019},
  url={https://proceedings.neurips.cc/paper/2019/hash/103303dd56a731e377d01f6a37badae3-Abstract.html}
}

@inproceedings{chami2019hyperbolic,
  title={Hyperbolic Graph Convolutional Neural Networks},
  author={Chami, Ines and Ying, Zhitao and R{\'e}, Christopher and Leskovec, Jure},
  booktitle={Advances in Neural Information Processing Systems},
  year={2019},
  url={https://proceedings.neurips.cc/paper/2019/hash/0415740eaa4d9decbc8da001d3fd805f-Abstract.html}
}

@inproceedings{xiong2022pseudo,
  title={Pseudo-Riemannian Graph Convolutional Networks},
  author={Xiong, Bo and Zhu, Shichao and Potyka, Nico and Pan, Shirui and Zhou, Chuan and Staab, Steffen},
  booktitle={Advances in Neural Information Processing Systems},
  year={2022},
  url={https://proceedings.neurips.cc/paper_files/paper/2022/hash/16c628ab12dc4caca8e7712affa6c767-Abstract-Conference.html}
}

@inproceedings{grover2025spectro,
  title={Spectro-Riemannian Graph Neural Networks},
  author={Grover, Karish and Yu, Haiyang and Song, Xiang and Zhu, Qi and Xie, Han and Ioannidis, Vassilis N. and Faloutsos, Christos},
  booktitle={International Conference on Learning Representations},
  year={2025},
  url={https://arxiv.org/abs/2502.00401}
}

@inproceedings{guo2024graphmore,
  title={Graphmore: Mitigating topological heterogeneity via mixture of riemannian experts},
  author={Guo, Zihao and Sun, Qingyun and Yuan, Haonan and Fu, Xingcheng and Zhou, Min and Gao, Yisen and Li, Jianxin},
  booktitle={Proceedings of the AAAI Conference on Artificial Intelligence},
  volume={39},
  number={11},
  pages={11754--11762},
  year={2025}
}

@article{ollivier2009ricci,
  title={Ricci Curvature of Markov Chains on Metric Spaces},
  author={Ollivier, Yann},
  journal={Journal of Functional Analysis},
  volume={256},
  number={3},
  pages={810--864},
  year={2009},
  publisher={Elsevier},
  doi={10.1016/j.jfa.2008.11.001},
  url={https://arxiv.org/abs/math/0701886}
}

@article{forman2003bochner,
  title={Bochner's Method for Cell Complexes and Combinatorial Ricci Curvature},
  author={Forman, Robin},
  journal={Discrete \& Computational Geometry},
  volume={29},
  number={3},
  pages={323--374},
  year={2003},
  doi={10.1007/s00454-002-0743-x},
  url={https://link.springer.com/article/10.1007/s00454-002-0743-x}
}

@inproceedings{ni2015ricci,
  title={Ricci Curvature of the Internet Topology},
  author={Ni, Chien-Chun and Lin, Yu-Yao and Gao, Jie and Gu, Xianfeng David and Saucan, Emil},
  booktitle={IEEE Conference on Computer Communications},
  pages={2758--2766},
  year={2015},
  organization={IEEE},
  url={https://arxiv.org/abs/1501.04138}
}

@inproceedings{zhao2024all,
  title={All in One and One for All: A Simple yet Effective Method towards Cross-domain Graph Pretraining},
  author={Zhao, Haihong and Chen, Aochuan and Sun, Xiangguo and Cheng, Hong and Li, Jia},
  booktitle={Proceedings of the 30th ACM SIGKDD Conference on Knowledge Discovery and Data Mining},
  year={2024},
  url={https://arxiv.org/abs/2402.09834}
}

@article{yu2024textfree,
  title={Text-Free Multi-domain Graph Pre-training: Toward Graph Foundation Models},
  author={Yu, Xingtong and Zhou, Chang and Fang, Yuan and Zhang, Xinming},
  journal={arXiv preprint arXiv:2405.13934},
  year={2024},
  eprint={2405.13934},
  archivePrefix={arXiv},
  url={https://arxiv.org/abs/2405.13934}
}

@inproceedings{wang2025multidomain,
  title={Multi-Domain Graph Foundation Models: Robust Knowledge Transfer via Topology Alignment},
  author={Wang, Shuo and Wang, Bokui and Shen, Zhixiang and Deng, Boyan and Kang, Zhao},
  booktitle={International Conference on Machine Learning},
  year={2025},
  url={https://arxiv.org/abs/2502.02017}
}

@inproceedings{yu2025samgpt,
  title={SAMGPT: Text-free Graph Foundation Model for Multi-domain Pre-training and Cross-domain Adaptation},
  author={Yu, Xingtong and Gong, Zechuan and Zhou, Chang and Fang, Yuan and Zhang, Hui},
  booktitle={Proceedings of the ACM Web Conference},
  year={2025},
  doi={10.1145/3696410.3714828},
  url={https://arxiv.org/abs/2502.05424}
}

@inproceedings{shi2026sa2gfm,
  title={SA$^2$GFM: Enhancing Robust Graph Foundation Models with Structure-Aware Semantic Augmentation},
  author={Shi, Junhua and Sun, Qingyun and Yuan, Haonan and Fu, Xingcheng},
  booktitle={Proceedings of the AAAI Conference on Artificial Intelligence},
  volume={40},
  number={18},
  pages={15716--15724},
  year={2026},
  doi={10.1609/aaai.v40i18.38602},
  url={https://doi.org/10.1609/aaai.v40i18.38602}
}

@inproceedings{sun2026graphglue,
  title={Multi-Domain Riemannian Graph Gluing for Building Graph Foundation Models},
  author={Sun, Li and Huang, Zhenhao and Chen, Silei and Yang, Lanxu and Ye, Junda and Su, Sen and Yu, Philip S.},
  booktitle={International Conference on Learning Representations},
  year={2026},
  url={https://openreview.net/forum?id=G3uNHQpP7J}
}

@inproceedings{kim2024carte,
  title={CARTE: Pretraining and Transfer for Tabular Learning},
  author={Kim, Myung Jun and Grinsztajn, Leo and Varoquaux, Gael},
  booktitle={International Conference on Machine Learning},
  pages={23843--23866},
  year={2024},
  organization={PMLR}
}

@article{fey2023relational,
  title={Relational Deep Learning: Graph Representation Learning on Relational Databases},
  author={Fey, Matthias and Hu, Weihua and Huang, Kexin and Lenssen, Jan Eric and Ranjan, Rishabh and Robinson, Joshua and Ying, Rex and You, Jiaxuan and Leskovec, Jure},
  journal={arXiv preprint arXiv:2312.04615},
  year={2023},
  eprint={2312.04615},
  archivePrefix={arXiv},
  url={https://arxiv.org/abs/2312.04615}
}

@inproceedings{mao2024position,
  title={Position: Graph Foundation Models Are Already Here},
  author={Mao, Haitao and Chen, Zhikai and Tang, Wenzhuo and Zhao, Jianan and Ma, Yao and Zhao, Tong and Shah, Neil and Galkin, Mikhail and Tang, Jiliang},
  booktitle={Proceedings of the 41st International Conference on Machine Learning},
  year={2024}
}

@article{morris2020tudataset,
  title={TUDataset: A Collection of Benchmark Datasets for Learning with Graphs},
  author={Morris, Christopher and Kriege, Nils M. and Bause, Franka and Kersting, Kristian and Mutzel, Petra and Neumann, Marion},
  journal={arXiv preprint arXiv:2007.08663},
  year={2020},
  eprint={2007.08663},
  archivePrefix={arXiv},
  url={https://arxiv.org/abs/2007.08663}
}

@article{wu2018moleculenet,
  title={MoleculeNet: A Benchmark for Molecular Machine Learning},
  author={Wu, Zhenqin and Ramsundar, Bharath and Feinberg, Evan N. and Gomes, Joseph and Geniesse, Caleb and Pappu, Aneesh S. and Leswing, Karl and Pande, Vijay},
  journal={Chemical Science},
  volume={9},
  number={2},
  pages={513--530},
  year={2018},
  publisher={Royal Society of Chemistry},
  doi={10.1039/C7SC02664A},
  url={https://pubs.rsc.org/en/content/articlelanding/2018/sc/c7sc02664a}
}

@inproceedings{schlichtkrull2018modeling,
  title     = {Modeling Relational Data with Graph Convolutional Networks},
  author    = {Schlichtkrull, Michael and Kipf, Thomas N. and Bloem, Peter and van den Berg, Rianne and Titov, Ivan and Welling, Max},
  booktitle = {European Semantic Web Conference},
  pages     = {593--607},
  year      = {2018},
  publisher = {Springer}
}

@inproceedings{zhang2021lorentzian,
  title     = {Lorentzian Graph Convolutional Networks},
  author    = {Zhang, Yiding and Wang, Xiao and Shi, Chuan and Liu, Nian and Song, Guojie},
  booktitle = {Proceedings of The Web Conference},
  pages     = {1249--1261},
  year      = {2021}
}

@article{li2022curvature,
  title   = {Curvature Graph Neural Network},
  author  = {Li, Haifeng and Cao, Jun and Zhu, Jiawei and Liu, Yu and Zhu, Qing and Wu, Guohua},
  journal = {Information Sciences},
  volume  = {592},
  pages   = {50--66},
  year    = {2022}
}

@inproceedings{robinson2024relbench,
  title={RelBench: A Benchmark for Deep Learning on Relational Databases},
  author={Robinson, Joshua and Ranjan, Rishabh and Hu, Weihua and Huang, Kexin and Han, Jiaqi and Dobles, Alejandro and Fey, Matthias and Lenssen, Jan Eric and Yuan, Yiwen and Zhang, Zecheng and He, Xinwei and Leskovec, Jure},
  booktitle={Advances in Neural Information Processing Systems, Datasets and Benchmarks Track},
  year={2024},
  url={https://proceedings.neurips.cc/paper_files/paper/2024/hash/25cd345233c65fac1fec0ce61d0f7836-Abstract-Datasets_and_Benchmarks_Track.html}
}

@article{bronstein2021geometric,
  title={Geometric Deep Learning: Grids, Groups, Graphs, Geodesics, and Gauges},
  author={Bronstein, Michael M. and Bruna, Joan and Cohen, Taco and Veli{\v{c}}kovi{\'c}, Petar},
  journal={arXiv preprint arXiv:2104.13478},
  year={2021},
  eprint={2104.13478},
  archivePrefix={arXiv},
  url={https://arxiv.org/abs/2104.13478}
}

@article{bachmann2020constant,
  title={Constant Curvature Graph Convolutional Networks},
  author={Bachmann, Gregor and B{\'e}cigneul, Gary and Ganea, Octavian-Eugen},
  journal={Proceedings of Machine Learning Research},
  volume={119},
  pages={486--496},
  year={2020},
  url={https://proceedings.mlr.press/v119/bachmann20a.html}
}

@inproceedings{gu2019learning,
  title={Learning Mixed-Curvature Representations in Product Spaces},
  author={Gu, Albert and Sala, Frederic and Gunel, Beliz and R{\'e}, Christopher},
  booktitle={International Conference on Learning Representations},
  year={2019},
  url={https://openreview.net/forum?id=HJxeWnCcF7}
}

@inproceedings{skopek2020mixed,
  title={Mixed-curvature Variational Autoencoders},
  author={Skopek, Ondrej and Ganea, Octavian-Eugen and B{\'e}cigneul, Gary},
  booktitle={International Conference on Learning Representations},
  year={2020},
  url={https://openreview.net/forum?id=S1g6xeSKDS}
}

@article{wang2025graphfoundation,
  title={Graph Foundation Models: A Comprehensive Survey},
  author={Wang, Zehong and Liu, Zheyuan and Ma, Tianyi and Li, Jiazheng and Zhang, Zheyuan and Fu, Xingbo and Li, Yiyang and Yuan, Zhengqing and Song, Wei and Ma, Yijun and Zeng, Qingkai and Chen, Xiusi and Zhao, Jianan and Li, Jundong and Jiang, Meng and Li{\`o}, Pietro and Chawla, Nitesh and Zhang, Chuxu and Ye, Yanfang},
  journal={arXiv preprint arXiv:2505.15116},
  year={2025},
  eprint={2505.15116},
  archivePrefix={arXiv},
  url={https://arxiv.org/abs/2505.15116}
}

@inproceedings{li2024pcconv,
  title={Pc-conv: Unifying homophily and heterophily with two-fold filtering},
  author={Li, Bingheng and Pan, Erlin and Kang, Zhao},
  booktitle={Proceedings of the AAAI conference on artificial intelligence},
  volume={38},
  number={12},
  pages={13437--13445},
  year={2024}
}

@inproceedings{guo2025graphmore,
  title={Graphmore: Mitigating topological heterogeneity via mixture of riemannian experts},
  author={Guo, Zihao and Sun, Qingyun and Yuan, Haonan and Fu, Xingcheng and Zhou, Min and Gao, Yisen and Li, Jianxin},
  booktitle={Proceedings of the AAAI Conference on Artificial Intelligence},
  volume={39},
  number={11},
  pages={11754--11762},
  year={2025}
}

@article{zhang2021hyperbolic,
  title={Hyperbolic graph attention network},
  author={Zhang, Yiding and Wang, Xiao and Shi, Chuan and Jiang, Xunqiang and Ye, Yanfang},
  journal={IEEE Transactions on Big Data},
  volume={8},
  number={6},
  pages={1690--1701},
  year={2021},
  publisher={IEEE}
}

@inproceedings{chen2022fully,
  title={Fully hyperbolic neural networks},
  author={Chen, Weize and Han, Xu and Lin, Yankai and Zhao, Hexu and Liu, Zhiyuan and Li, Peng and Sun, Maosong and Zhou, Jie},
  booktitle={Proceedings of the 60th Annual Meeting of the Association for Computational Linguistics (Volume 1: Long Papers)},
  pages={5672--5686},
  year={2022}
}

@book{burago2001course,
  title={A Course in Metric Geometry},
  author={Burago, Dmitri and Burago, Yuri and Ivanov, Sergei},
  series={Graduate Studies in Mathematics},
  volume={33},
  publisher={American Mathematical Society},
  address={Providence, RI},
  year={2001},
  isbn={9780821821299},
  url={https://pubs.ams.org/ebooks/gsm/033/}
}

@book{bridson1999metric,
  title={Metric Spaces of Non-Positive Curvature},
  author={Bridson, Martin R. and Haefliger, Andr{\'e}},
  series={Grundlehren der mathematischen Wissenschaften},
  volume={319},
  publisher={Springer},
  address={Berlin, Heidelberg},
  year={1999},
  isbn={9783540643241},
  doi={10.1007/978-3-662-12494-9},
  url={https://link.springer.com/book/10.1007/978-3-662-12494-9}
}

@article{joanes1998comparing,
  title={Comparing Measures of Sample Skewness and Kurtosis},
  author={Joanes, D. N. and Gill, C. A.},
  journal={Journal of the Royal Statistical Society: Series D (The Statistician)},
  volume={47},
  number={1},
  pages={183--189},
  year={1998},
  publisher={Wiley},
  doi={10.1111/1467-9884.00122},
  url={https://academic.oup.com/jrsssd/article-abstract/47/1/183/7123309}
}

@article{bourgain1985lipschitz,
  title={On Lipschitz Embedding of Finite Metric Spaces in Hilbert Space},
  author={Bourgain, Jean},
  journal={Israel Journal of Mathematics},
  volume={52},
  number={1-2},
  pages={46--52},
  year={1985},
  publisher={Springer},
  doi={10.1007/BF02776078}
}

@article{linial1995geometry,
  title={The Geometry of Graphs and Some of Its Algorithmic Applications},
  author={Linial, Nathan and London, Eran and Rabinovich, Yuri},
  journal={Combinatorica},
  volume={15},
  number={2},
  pages={215--245},
  year={1995},
  publisher={Springer},
  doi={10.1007/BF01200757}
}

@inproceedings{nr,
     title={The Network Data Repository with Interactive Graph Analytics and Visualization},
     author={Ryan A. Rossi and Nesreen K. Ahmed},
     booktitle={AAAI},
     url={https://networkrepository.com},
     year={2015}
}

@article{Zhou2022,
  title={TeleGraph: A Benchmark Dataset for Hierarchical Link Prediction},
  author={Zhou, Min and Li, Bisheng and Yang, Menglin and Pan, Lujia},
  year={2022},
  journal={arXiv (Cornell University)},
  volume={abs/2204.07703},
  doi={10.48550/arxiv.2204.07703},
}

@article{Yang2022,
  title={Revisiting Semi-Supervised Learning with Graph Embeddings},
  author={Yang, Zhilin and Cohen, William and Salakhutdinov, Ruslan},
  year={2022},
  journal={arXiv (Cornell University)},
  doi={10.48550/arxiv.1603.08861},
}

@article{Pei2020,
  title={Geom-GCN: Geometric Graph Convolutional Networks},
  author={Pei, Hongbin and Wei, Bingzhe and Chang, Kevin and Lei, Yu and Yang, Bo},
  year={2020},
  journal={ICLR},
}

@inproceedings{muggleton1997predictive,
  title={The predictive toxicology evaluation challenge},
  author={Muggleton, SH},
  booktitle={IJCAI-97: Proceedings of the Fifteenth International Joint Conference on Artificial Intelligence, Nagoya, Japan, August 23-29, 1997},
  volume={2},
  pages={4},
  year={1997},
  organization={Morgan Kaufmann}
}

@article{mutlu2016policy,
  title={Policy-based memoization for ILP-based concept discovery systems},
  author={Mutlu, Alev and Karagoz, Pinar},
  journal={Journal of Intelligent Information Systems},
  volume={46},
  number={1},
  pages={99--120},
  year={2016},
  publisher={Springer}
}

@article{spearman1904proof,
  title     = {The Proof and Measurement of Association between Two Things},
  author    = {Spearman, Charles},
  journal   = {The American Journal of Psychology},
  volume    = {15},
  number    = {1},
  pages     = {72--101},
  year      = {1904},
  doi       = {10.2307/1412159}
}

@article{kendall1938new,
  title     = {A New Measure of Rank Correlation},
  author    = {Kendall, Maurice G.},
  journal   = {Biometrika},
  volume    = {30},
  number    = {1/2},
  pages     = {81--93},
  year      = {1938},
  doi       = {10.1093/biomet/30.1-2.81}
}

@article{jaccard1901etude,
  title     = {{\'E}tude comparative de la distribution florale dans une portion des Alpes et du Jura},
  author    = {Jaccard, Paul},
  journal   = {Bulletin de la Soci{\'e}t{\'e} Vaudoise des Sciences Naturelles},
  volume    = {37},
  pages     = {547--579},
  year      = {1901}
}

@article{battaglia2018relational,
  title={Relational inductive biases, deep learning, and graph networks},
  author={Battaglia, Peter W and Hamrick, Jessica B and Bapst, Victor and Sanchez-Gonzalez, Alvaro and Zambaldi, Vinicius and Malinowski, Mateusz and Tacchetti, Andrea and Raposo, David and Santoro, Adam and Faulkner, Ryan and others},
  journal={arXiv preprint arXiv:1806.01261},
  year={2018}
}

@inproceedings{lv2021we,
  title={Are we really making much progress? revisiting, benchmarking and refining heterogeneous graph neural networks},
  author={Lv, Qingsong and Ding, Ming and Liu, Qiang and Chen, Yuxiang and Feng, Wenzheng and He, Siming and Zhou, Chang and Jiang, Jianguo and Dong, Yuxiao and Tang, Jie},
  booktitle={Proceedings of the 27th ACM SIGKDD conference on knowledge discovery \& data mining},
  pages={1150--1160},
  year={2021}
}

@article{asif2021graph,
  title={Graph neural network: A comprehensive review on non-euclidean space},
  author={Asif, Nurul A and Sarker, Yeahia and Chakrabortty, Ripon K and Ryan, Michael J and Ahamed, Md Hafiz and Saha, Dip K and Badal, Faisal R and Das, Sajal K and Ali, Md Firoz and Moyeen, Sumaya I and others},
  journal={Ieee Access},
  volume={9},
  pages={60588--60606},
  year={2021},
  publisher={IEEE}
}

@article{ye2025mose,
  title={MoSE: Unveiling Structural Patterns in Graphs via Mixture of Subgraph Experts},
  author={Ye, Junda and Zhang, Zhongbao and Sun, Li and Luo, Siqiang},
  journal={arXiv preprint arXiv:2509.09337},
  year={2025}
}

@inproceedings{du2026graphoracle,
  title={GraphOracle: Efficient fully-inductive knowledge graph reasoning via relation-dependency graphs},
  author={Du, Enjun and Liu, Siyi and Zhang, Yongqi},
  booktitle={Proceedings of the AAAI Conference on Artificial Intelligence},
  volume={40},
  number={23},
  pages={19055--19063},
  year={2026}
}

@article{eremeev2025turning,
  title={Turning tabular foundation models into graph foundation models},
  author={Eremeev, Dmitry and Bazhenov, Gleb and Platonov, Oleg and Babenko, Artem and Prokhorenkova, Liudmila},
  journal={arXiv preprint arXiv:2508.20906},
  year={2025}
}

@inproceedings{wang2025cooperation,
  title={Cooperation of Experts: Fusing Heterogeneous Information with Large Margin},
  author={Wang, Shuo and Huang, Shunyang and Yuan, Jinghui and Shen, Zhixiang and Kang, Zhao},
  booktitle={International Conference on Machine Learning},
  pages={63169--63185},
  year={2025},
  organization={PMLR}
}

@article{sun2026survey,
  title={A Survey on Foundation Models for Structured Data: Tabular, Time Series, and Graphs},
  author={Sun, Qingyun and Yuan, Haonan and Huang, Yi and Zhang, Ziwei and Fu, Xingcheng and Wang, Ruijie and Zhou, Haoyi and Wu, Jia and Li, Jianxin and Yu, Philip S},
  year={2026},
  publisher={Preprints}
}


\newpage
\tableofcontents
\newpage
\appendix

\section{Baselines}

\subsection{Topology-Agnostic and Euclidean Baselines}

\textbf{MLP.}
We first include a multi-layer perceptron (MLP) as a topology-agnostic baseline. MLP uses node features as input and does not perform message passing over the graph. In the context of our curvature-aware evaluation, MLP plays an important diagnostic role: it measures how much predictive signal can be explained by attributes alone, without using the graph metric induced by edges. If MLP performs competitively on a dataset, the task may be dominated by feature separability rather than geometric structure. Conversely, a substantial gap between MLP and topology-aware models indicates that relational information is necessary. Thus, MLP serves as a curvature-null control that helps separate feature-level effects from geometry-induced inductive bias.

\textbf{GCN.}
Graph Convolutional Network (GCN) is one of the most widely used Euclidean message-passing baselines \citep{kipf2017semi}. It propagates information through a normalized adjacency operator and performs feature smoothing over local neighborhoods. From a geometric perspective, GCN assumes that neighbor aggregation in a flat Euclidean representation space is sufficient to capture the task-relevant graph structure. This makes GCN a natural reference point for evaluating whether non-Euclidean models provide additional benefits beyond standard Euclidean smoothing. Since GCN is also known to behave similarly to a low-pass graph filter, its performance across curvature regimes helps reveal whether flat local aggregation is robust to graphs with positive, negative, or near-zero curvature.

\textbf{GAT.}
Graph Attention Network (GAT) extends Euclidean message passing by learning attention weights over neighboring nodes \citep{velickovic2018graph}. Instead of assigning fixed normalized weights as in GCN, GAT adaptively determines the relative importance of each neighbor during aggregation. This improves the flexibility of local message passing but still operates in Euclidean space. In our benchmark, GAT is used to test whether adaptive neighbor weighting alone can compensate for curvature mismatch. If GAT improves over GCN within a curvature regime, the gain can be attributed to local anisotropic aggregation; if it remains inferior to geometry-aware models, this suggests that attention in a flat space is not sufficient to capture the underlying graph metric.

\textbf{GraphSAGE.}
GraphSAGE is an inductive Euclidean GNN that generates node representations by sampling and aggregating features from local neighborhoods \citep{hamilton2017inductive}. Compared with transductive embedding methods, GraphSAGE learns an aggregation function that can generalize to unseen nodes or graphs. We include GraphSAGE because it represents a widely adopted scalable message-passing paradigm and provides a strong Euclidean baseline for both effectiveness and efficiency. In the curvature-aware setting, GraphSAGE helps evaluate whether inductive neighborhood aggregation can remain stable across different geometric regimes, especially when the local graph metric varies significantly across datasets.

\subsection{Spectral and Homophily/Heterophily-Aware Euclidean Models}

In addition to classical spatial GNNs, we include spectral or filter-based Euclidean models to distinguish curvature effects from frequency-domain effects. These methods do not explicitly change the representation manifold, but they modify how graph signals are propagated across the spectrum.

\textbf{PCNet.}
PCNet instantiates PC-Conv, a two-fold filtering mechanism designed to unify homophilic and heterophilic aggregation \citep{li2024pcconv}. It combines a local low-pass filter with a heterophilic graph heat kernel and uses Poisson-Charlier polynomial filters to capture multi-order information. Although PCNet remains Euclidean, it is important for \textsc{CurvBench} because it addresses a different structural axis from curvature: the compatibility between graph topology and label smoothness. Including PCNet allows us to examine whether performance variations attributed to curvature can instead be explained by homophily, heterophily, or spectral filtering flexibility. This distinction is crucial because curvature and homophily are correlated in some datasets but are not equivalent geometric properties.

Hyperbolic models are designed for negatively curved spaces, where volume grows exponentially with radius. Such geometry is well suited to graphs with hierarchical, tree-like, or scale-free structures. These models directly embody the hypothesis that representation geometry should match graph geometry, making them central to our benchmark.

\textbf{HGNN.}
Hyperbolic Graph Neural Network (HGNN) generalizes GNN operations to Riemannian manifolds through differentiable exponential and logarithmic maps \citep{liu2019hyperbolic}. Its message passing is performed by mapping node representations between the hyperbolic manifold and tangent spaces. In \textsc{CurvBench}, HGNN serves as an early representative of manifold-based graph learning. It tests whether replacing Euclidean hidden representations with hyperbolic ones improves performance on negatively curved graphs, and whether such benefits persist outside the regime for which hyperbolic geometry is theoretically motivated.

\textbf{HGCN.}
Hyperbolic Graph Convolutional Network (HGCN) extends GCN to the hyperboloid model of hyperbolic space \citep{chami2019hyperbolic}. It maps Euclidean input features into hyperbolic representations, performs neighborhood aggregation through tangent-space operations, and learns layer-wise curvature parameters. HGCN is particularly relevant to our benchmark because its design explicitly targets hierarchical and scale-free graphs. Evaluating HGCN across positive-, negative-, and near-zero-curvature regimes allows us to test whether its reported advantage is tied to negative curvature or whether it generalizes as a broader graph representation mechanism.

\textbf{HAT.}
Hyperbolic Graph Attention Network (HAT) transfers the attention mechanism from Euclidean space to hyperbolic geometry \citep{zhang2021hyperbolic}. It uses gyrovector operations and hyperbolic proximity to define feature transformation and attention-based aggregation. Compared with HGCN, HAT allows us to isolate the role of attention under a negative-curvature geometry. In our benchmark, HAT helps answer whether hyperbolic attention provides additional benefits over Euclidean attention, and whether such benefits are concentrated in negatively curved graphs.

\textbf{HyboNet.}
HyboNet is a fully hyperbolic neural architecture based on the Lorentz model \citep{chen2022fully}. Unlike methods that implement most operations in tangent spaces, HyboNet formalizes core neural operations directly in hyperbolic space through Lorentz transformations. This design reduces the reliance on repeated logarithmic and exponential mappings and provides a stronger test of fully hyperbolic inductive bias. In \textsc{CurvBench}, HyboNet is included to evaluate whether more faithful hyperbolic computation improves robustness and efficiency across curvature regimes.

\subsection{Pseudo- and Mixed-Curvature Models}

Real-world graphs often contain multiple geometric patterns simultaneously. A single constant-curvature space may be too restrictive for graphs with both hierarchical and cyclic substructures~\citep{zhang2021lorentzian, li2022curvature}. We therefore include pseudo- and mixed-curvature models that aim to represent heterogeneous topology beyond purely Euclidean or purely hyperbolic assumptions.

\textbf{QGCN.}
Pseudo-Riemannian Graph Convolutional Network (QGCN) extends GCNs to pseudo-Riemannian manifolds with indefinite metrics \citep{xiong2022pseudo}. These manifolds generalize hyperbolic and spherical geometries and provide a flexible representation space for graphs with mixed topologies, such as hierarchical structures with cycles. QGCN is central to our benchmark because it directly challenges the single-curvature assumption: if a graph contains both positive- and negative-curvature regions, a pseudo-Riemannian representation may be more appropriate than a purely hyperbolic or Euclidean one.

\textbf{CUSP.}
CUSP is a mixed-curvature spectral GNN that integrates curvature and spectral information \citep{grover2025spectro}. It introduces a curvature-aware graph Laplacian based on Ollivier-Ricci curvature, applies mixed-curvature graph filters over product manifolds, and uses curvature-based positional encoding for hierarchical pooling. We include CUSP because it explicitly connects discrete graph curvature with spectral filtering, thereby addressing two dimensions of graph structure that are often studied separately. In our benchmark, CUSP helps test whether curvature-aware filtering improves performance in regimes where both geometric and frequency-domain signals are important.

\subsection{Adaptive Riemannian Models}

Beyond fixed or globally mixed curvature, adaptive Riemannian models attempt to assign different geometric spaces to different nodes or local regions. This is particularly aligned with the motivation of \textsc{CurvBench}, where graph datasets may contain construction-induced or locally varying curvature patterns.

\textbf{GraphMoRE.}
GraphMoRE introduces a Mixture of Riemannian Experts framework to mitigate topological heterogeneity \citep{guo2025GraphMoRE}. Instead of embedding all nodes into the same global manifold, it uses a topology-aware gating mechanism to route nodes to different Riemannian experts and constructs personalized mixed-curvature spaces. This model is especially important for our benchmark because it operationalizes the idea that curvature may vary locally within a graph. Its performance across regimes provides evidence for whether adaptive curvature selection is necessary when graphs contain heterogeneous or mixed geometric structures.

\subsection{Graph Foundation Models}

Graph foundation models (GFMs) have recently emerged as a prominent paradigm in graph learning~\citep{mao2024position}. Their central goal is to capture domain-invariant knowledge from graphs across diverse domains and transfer such knowledge to unseen graph domains. Existing GFMs differ substantially in the type of transferable signal they emphasize. Methods such as GCOPE~\citep{zhao2024all} and SAMGPT~\citep{yu2025samgpt} primarily focus on feature-level or prompt-based alignment, whereas approaches such as MDGFM~\citep{wang2025multidomain} and GraphGluing~\citep{sun2026graphglue} place greater emphasis on intrinsic topological properties and structure-level transfer. In this sense, different GFMs encode inductive biases from different perspectives. Therefore, rather than categorizing each GFM from a predefined geometric viewpoint, we treat them as representative foundation-model baselines and empirically examine whether their claimed cross-domain transferability truly holds across different geometry regimes. Through extensive curvature-stratified experiments, we provide a rigorous evaluation of whether GFMs can generalize uniformly across graphs with distinct intrinsic geometries.

\subsection{Discussion: Why These Models?}

In summary, our model selection reflects a curated spectrum of geometric assumptions rather than a mere collection of popular baselines. MLP removes graph geometry entirely; GCN, GAT, GraphSAGE, and PCNet operate in Euclidean space with different aggregation or filtering mechanisms; HGNN, HGCN, HAT, and HyboNet instantiate negative-curvature inductive bias; QGCN and CUSP model pseudo- or mixed-curvature structures; GraphMoRE adapts curvature at a finer granularity; and graph foundation models test whether pretrained graph knowledge transfers across curvature regimes. This taxonomy adheres to the core evaluation principle of \textsc{CurvBench}: model performance must be interpreted through the lens of graph geometry. Consequently, rather than seeking a universally superior model, our analysis investigates the specific curvature and task regimes where each inductive bias excels.

\section{Detailed Description of Datasets}\label{datasplits}

To comprehensively evaluate various methods, we conduct experiments on a wide variety of public benchmark graph datasets. As outlined below, these datasets encompass diverse semantic contexts and topological structures:

\begin{itemize}

\item \textbf{Citation Networks (\textit{Cora}, \textit{Citeseer}, \textit{PubMed})}: Standard academic citation graphs where nodes represent scientific papers and edges denote undirected citations. Node features are bag-of-words representations of the documents, and labels correspond to academic subfields.

\item \textbf{Webpage and Wikipedia Networks (\textit{Cornell}, \textit{Actor})}: \textit{Cornell} consists of web pages from Cornell University with hyperlinks as edges. \textit{Actor} is an actor co-occurrence network constructed from Wikipedia pages, which typically exhibits distinct heterophilic properties.

\item \textbf{Domain-Specific and Social Networks}:
    \begin{itemize}
        \item \textit{Airport}: A transportation network where nodes are airports and edges indicate flights between them.
        \item \textit{Disease}: An epidemiological network designed to simulate disease propagation dynamics.
        \item \textit{Telecom}: A telecommunications network modeling interactions and connectivity between customer entities.
        \item \textit{CS\_Phds}: An academic social network describing computer science PhD students and their advising/collaboration relationships.
    \end{itemize}
\end{itemize}

\paragraph{Data Preprocessing Details.} 
Given the varied sources of the raw data, all datasets are systematically structured into a unified standard tensor format (e.g., PyTorch Geometric \texttt{Data} objects) containing node feature matrices $\mathbf{X}$, adjacency edge indices $\mathbf{A}$, and label vectors $\mathbf{y}$. Due to scale contradictions in raw feature measurements, we enforce a global row-wise feature normalization technique seamlessly across all samples to prevent numerical overflows and computational instability during topological metrics calculations.

\paragraph{Label Processing and Splitting Protocol.}
Translating specialized datasets such as \textbf{\textit{Airport}} and \textbf{\textit{CS\_Phds}} into robust categorical evaluation frames inherently requires addressing domain-specific challenges through several crucial data transformation procedures:
\begin{itemize}
    \item \textbf{\textit{Airport} Data Treatment:} This raw transportation network intrinsically comprises hundreds of unverified or metadata-missing airports mapped strictly to a dummy label class (i.e., $-1$). During the split sampling formulation, these non-labeled entities are unequivocally systematically purged from the label masking pools. They function exclusively as structural connective bridges propagating topological message passing but explicitly evade calculating classification loss or generalizability measurements.
    \item \textbf{\textit{CS\_Phds} Quartile Discretization:} In contrast to standard graphs, \textit{CS\_Phds} fundamentally originates as a continuous regression task predicting researchers' academic scalability metrics. To seamlessly interoperate within our 4-class node classification benchmark, we executed targeted uniform quartile discretization on its labels. By stratifying the inherent continuous values, we coerced the distribution into 4 distinct, relatively perfectly balanced numeric intervals (each absorbing roughly 250 samples uniformly).
\end{itemize}

Upon rectifying the labels array layout correspondingly, a tailored dataset stratification was implemented on \textit{Airport} and \textit{CS\_Phds}. Since neither offers predefined independent splits, their training subsets are constituted adhering precisely to an internally seeded \textit{stratified dynamic K-shot sampling schema}. Remaining categorized nodes from each class are subsequently divided orthogonally between isolated \texttt{validation} and \texttt{test} boolean masks.

\paragraph{Construction of the Telecommunication Network (\textit{Telecom}).}
The \textit{Telecom} dataset represents a real-world physical communication infrastructure, which is programmatically constructed from multi-relational tabular logs. In this network, nodes correspond to distinct telecommunication network elements and hardware equipment, such as base stations (\texttt{NODEB}), microwave transmission devices (\texttt{MICROWAVE}), and core routers (\texttt{ROUTER}). The edges are established based on the physical cabling or logical communication links that connect these network elements in the relational databases. To formulate the node features, equipment-specific configurations and heterogeneous categorical attributes (spanning up to 240 property columns in the raw tabular data) are sequentially extracted, numerically encoded, and flattened into 240-dimensional continuous feature vectors for each node. Ultimately, the tabular records and their interrelated equipment dependencies are projected into a unified graph structure, yielding a highly sparse infrastructure topology comprising 41,143 nodes and 41,424 edges.

\paragraph{Automated Table-to-Graph Construction Pipeline.}
To structurally bridge the gap between multi-relational databases and graph-based computational operations, we orchestrate an automated Table-to-Graph conversion framework that flawlessly preserves both topological schemas and tabular semantics. Guided by explicitly predefined Entity-Relationship (ER) mappings extracted from the original database architectures (e.g., matching foreign keys such as projecting chemical bonds to atoms, or biopsy records to specific patients), we project each isolated data table as a distinct node type within a PyTorch Geometric \texttt{HeteroData} object. For node feature initialization, we abstain from elementary hashing of table contents; instead, we deploy the CARTE framework, leveraging a pre-trained semantic language model (\textit{FastText}), to sequentially embed the heterogeneous row-level properties—spanning numerical, categorical, and textual columns—into continuous, dense feature representations. Subsequently, the heterogeneous graph topology is systematically instantiated by performing intersection alignments across primary and foreign keys, translating discrete tabular cross-references into explicitly weighted adjacency edge indexes seamlessly.

\paragraph{Data splits for GFMs.}
As shown in Table~\ref{tab:app-lp-results}, we adopt a geometry-balanced three-fold splitting protocol over the nine natural graph datasets for GFM evaluation. 
The datasets are first grouped into three curvature regimes: near-zero, positive, and negative. 
In each fold, we select two datasets from each geometry regime as source datasets for pre-training, and hold out the remaining dataset in the same regime for evaluation. 

As a result, every fold contains six pre-training datasets and three evaluation datasets, with exactly one held-out dataset from each curvature regime. 
This design ensures that the pre-training stage always observes geometrically diverse source graphs, while the evaluation stage tests whether the learned transferable representations generalize to unseen datasets under each curvature regime. 
Across the three folds, every dataset is used once as an evaluation dataset and twice as a pre-training dataset, yielding a balanced protocol for assessing geometry-conditioned transfer behavior.

\begin{table}[htbp]
\centering
\caption{Data splits for Graph Foundation Models.}
\label{tab:app-lp-results}
\resizebox{\linewidth}{!}{
\begin{tabular}{c|ccc|ccc|ccc}
\toprule
Dataset & Cora                        & Citeseer                    & OubMed                      & Airport                     & Cornell                     & Actor                       & Telecom                     & Disease                     & CS$\_$Phds                  \\
\midrule
Fold A  & \textcolor{red}{Evaluate}   & \textcolor{green!60!black}{Pretrain} & \textcolor{green!60!black}{Pretrain} & \textcolor{red}{Evaluate}   & \textcolor{green!60!black}{Pretrain} & \textcolor{green!60!black}{Pretrain} & \textcolor{red}{Evaluate}   & \textcolor{green!60!black}{Pretrain} & \textcolor{green!60!black}{Pretrain} \\
Fold B  & \textcolor{green!60!black}{Pretrain} & \textcolor{red}{Evaluate}   & \textcolor{green!60!black}{Pretrain} & \textcolor{green!60!black}{Pretrain} & \textcolor{red}{Evaluate}   & \textcolor{green!60!black}{Pretrain} & \textcolor{green!60!black}{Pretrain} & \textcolor{red}{Evaluate}   & \textcolor{green!60!black}{Pretrain} \\
Fold C  & \textcolor{green!60!black}{Pretrain} & \textcolor{green!60!black}{Pretrain} & \textcolor{red}{Evaluate}   & \textcolor{green!60!black}{Pretrain} & \textcolor{green!60!black}{Pretrain} & \textcolor{red}{Evaluate}   & \textcolor{green!60!black}{Pretrain} & \textcolor{green!60!black}{Pretrain} & \textcolor{red}{Evaluate}   \\
\bottomrule

\end{tabular}
}
\end{table}

\section{Implementation Details}
\label{implement}

To ensure a faithful and fair evaluation, we maintain consistent experimental protocols across all considered baselines. Unless otherwise specified, all results are reported as averages over five independent runs with distinct random seeds. Experiments are primarily conducted on a single NVIDIA A100 GPU. For conventional graph models, a 24GB GPU is sufficient for reproduction. However, due to the substantial memory overhead of the pre-training stage, GFM baselines are evaluated on 80GB configurations. While hardware transitions may introduce minor runtime variations, they do not impact the reported performance metrics or the resulting scientific conclusions. All methods are evaluated under identical hyperparameter settings whenever applicable to ensure parity. Comprehensive implementation details, including specific training protocols and step-by-step reproduction instructions, are provided in our released code repository.

\section{Proofs for Section~\ref{sec:theory}}
\label{app:theory-proofs}

\subsection{Proof of Theorem~\ref{thm:curvature-stability}}

\begin{proof}

For a generic metric space \((\mathcal X,d_{\mathcal X})\) and a valid quadruple
\(q=(a,b,c;m)\), write
\[
    x=d_{\mathcal X}(a,m),\qquad
    y=d_{\mathcal X}(b,c),\qquad
    z=d_{\mathcal X}(a,b),\qquad
    w=d_{\mathcal X}(a,c).
\]
Define
\[
    \Phi(x,y,z,w)
    =
    \frac{
        x^2+\frac14y^2-\frac12(z^2+w^2)
    }{
        2x
    }
    =
    \frac{x}{2}
    +
    \frac{y^2}{8x}
    -
    \frac{z^2}{4x}
    -
    \frac{w^2}{4x}.
\]
Then \(\xi_{\mathcal X}(q)=\Phi(x,y,z,w)\). Moreover, \(\Phi\) is homogeneous of
degree one:
\[
    \Phi(\lambda x,\lambda y,\lambda z,\lambda w)
    =
    \lambda\Phi(x,y,z,w).
\]

Let $(x_G,y_G,z_G,w_G)$ be the four graph distances and
$(x_M,y_M,z_M,w_M)$ be the corresponding distances after embedding into $\mathcal M$. By the
definition of $\delta=\mathrm{dis}_{\infty}^{\lambda}(h;G,\mathcal M)$,
\[
    |x_M-\lambda x_G|\le\delta,\quad
    |y_M-\lambda y_G|\le\delta,\quad
    |z_M-\lambda z_G|\le\delta,\quad
    |w_M-\lambda w_G|\le\delta.
\]
Because $q$ is valid, $x_G=d_G(a,m)\ge1$. Since $\delta<\lambda/2$,
\[
    x_M\ge \lambda x_G-\delta \ge \lambda-\delta > \lambda/2.
\]
Moreover, every finite graph distance is at most $D=\mathrm{diam}_f(G)$, so
\[
    x_M,y_M,z_M,w_M
    \le
    \lambda D+\delta
    \le
    \lambda(D+1/2).
\]

It remains to bound the Lipschitz constant of $\Phi$ on the region
\[
    x\ge\lambda/2,
    \qquad
    0\le x,y,z,w\le \lambda(D+1/2).
\]
The partial derivatives are
\[
    \frac{\partial \Phi}{\partial x}
    =
    \frac12
    -
    \frac{y^2}{8x^2}
    +
    \frac{z^2+w^2}{4x^2},
    \qquad
    \frac{\partial \Phi}{\partial y}
    =
    \frac{y}{4x},
\]
\[
    \frac{\partial \Phi}{\partial z}
    =
    -\frac{z}{2x},
    \qquad
    \frac{\partial \Phi}{\partial w}
    =
    -\frac{w}{2x}.
\]
On the above region,
\[
    \left|
        \frac{\partial \Phi}{\partial x}
    \right|
    \le
    \frac12+\frac58(2D+1)^2,
\]
and
\[
    \left|
        \frac{\partial \Phi}{\partial y}
    \right|
    +
    \left|
        \frac{\partial \Phi}{\partial z}
    \right|
    +
    \left|
        \frac{\partial \Phi}{\partial w}
    \right|
    \le
    \frac54(2D+1).
\]
Therefore, by the mean value theorem and the $\ell_{\infty}$ bound on the four distance errors,
\[
    \left|
        \Phi(x_M,y_M,z_M,w_M)
        -
        \Phi(\lambda x_G,\lambda y_G,\lambda z_G,\lambda w_G)
    \right|
    \le
    C_D\delta,
\]
where
\[
    C_D
    =
    \frac12
    +
    \frac58(2D+1)^2
    +
    \frac54(2D+1).
\]
Using homogeneity,
\[
    \Phi(\lambda x_G,\lambda y_G,\lambda z_G,\lambda w_G)
    =
    \lambda\Phi(x_G,y_G,z_G,w_G)
    =
    \lambda\xi_G(q).
\]
Hence,
\[
    \left|
        \xi_{\mathcal M}(h(q))-\lambda\xi_G(q)
    \right|
    \le
    C_D\delta.
\]
Taking expectation over \(q\sim\mathcal Q_G\) gives
Eq.~\eqref{eq:curv-exp-bound}.

Finally, let
\[
    D_{\mathcal H}^{\lambda}
    =
    \mathrm{Dist}_{\mathcal H}^{\lambda}(G,\mathcal M),
    \qquad
    \eta_{\mathcal H}^{\lambda}
    =
    \eta_{\mathcal H}^{\lambda}(G,\mathcal M).
\]
If \(D_{\mathcal H}^{\lambda}<\lambda/2\), then for any sufficiently small
\(\rho>0\), there exists \(h_{\rho}\in\mathcal H\) such that
\[
    \mathrm{dis}_{\infty}^{\lambda}(h_{\rho};G,\mathcal M)
    \le
    D_{\mathcal H}^{\lambda}+\rho
    <
    \lambda/2.
\]
Applying Eq.~\eqref{eq:curv-exp-bound} to \(h_{\rho}\) gives
\[
    \eta_{\mathcal H}^{\lambda}
    \le
    C_D
    \left(
        D_{\mathcal H}^{\lambda}+\rho
    \right).
\]
Letting \(\rho\to0\) yields
\[
    D_{\mathcal H}^{\lambda}
    \ge
    \frac{
        \eta_{\mathcal H}^{\lambda}
    }{
        C_D
    }.
\]
Thus proves the theorem.
\end{proof}

\subsection{Proof of Theorem~\ref{thm:partial-order-variance}}

\begin{proof}
Fix a model pair $(i,j)$ and write $p_r^z=p_{ij,r}^z$ for simplicity. If two graphs are drawn
independently from the same regime $r$, then the probability that their pairwise states agree is
\[
    \sum_{z\in\{-1,0,+1\}} (p_r^z)^2.
\]
Thus the within-regime disagreement probability for this pair is
\[
    D_{ij}^{\mathrm{within}}
    =
    1
    -
    \frac1K
    \sum_{r\in\mathcal R}
    \sum_{z\in\{-1,0,+1\}}
    (p_r^z)^2.
\]
If the two graphs are drawn independently from two different regimes, then the cross-regime
agreement probability is
\[
    \frac{1}{K(K-1)}
    \sum_{r\ne s}
    \sum_{z\in\{-1,0,+1\}}
    p_r^z p_s^z.
\]
Therefore
\[
    D_{ij}^{\mathrm{cross}}
    =
    1
    -
    \frac{1}{K(K-1)}
    \sum_{r\ne s}
    \sum_{z\in\{-1,0,+1\}}
    p_r^z p_s^z.
\]

Subtracting the two quantities gives
\[
    D_{ij}^{\mathrm{cross}}
    -
    D_{ij}^{\mathrm{within}}
    =
    \sum_{z\in\{-1,0,+1\}}
    \left[
        \frac1K\sum_r (p_r^z)^2
        -
        \frac{1}{K(K-1)}
        \sum_{r\ne s}p_r^z p_s^z
    \right].
\]
For each fixed state $z$, define
\[
    \bar p^z=\frac1K\sum_r p_r^z,
    \qquad
    \overline{(p^z)^2}
    =
    \frac1K\sum_r (p_r^z)^2.
\]
Since
\[
    \sum_{r\ne s}p_r^z p_s^z
    =
    \left(\sum_r p_r^z\right)^2
    -
    \sum_r (p_r^z)^2,
\]
we have
\[
    \frac1K\sum_r (p_r^z)^2
    -
    \frac{1}{K(K-1)}
    \sum_{r\ne s}p_r^z p_s^z
    =
    \frac{K}{K-1}
    \left(
        \overline{(p^z)^2}-(\bar p^z)^2
    \right).
\]
The term in parentheses is
$\mathrm{Var}_{r\sim\mathrm{Unif}(\mathcal R)}(p_{ij,r}^{z})$. Hence
\[
    D_{ij}^{\mathrm{cross}}
    -
    D_{ij}^{\mathrm{within}}
    =
    \frac{K}{K-1}
    \sum_{z\in\{-1,0,+1\}}
    \mathrm{Var}_{r}
    \left(
        p_{ij,r}^{z}
    \right).
\]
Finally, $d_{\epsilon}(G,G')$ is the average of pairwise disagreement indicators over all
$\binom{N}{2}$ model pairs. Averaging the pairwise identity over $i<j$ proves
Eq.~\eqref{eq:partial-order-gap}. The non-negativity follows from non-negativity of variance.
The gap is zero if and only if all these variances are zero, i.e., if and only if every pairwise
comparison-state distribution is invariant across regimes.
\end{proof}

\section{Details of Metrics}
\label{app:metrics}

This section details the two geometric metrics foundational to \textsc{CurvBench}: node-level sectional curvature and graph-level curvature skewness. These metrics characterize complementary dimensions of graph geometry. The sectional curvature metric quantifies the signed local deviation of the graph's manifold from Euclidean midpoint geometry, identifying regions of positive (spherical) or negative (hyperbolic) curvature, while the skewness metric measure captures the asymmetry of the resulting node-wise curvature distribution across the entire graph. Together, these quantities enable \textsc{CurvBench} to transcend aggregate leaderboards by stratifying datasets according to their intrinsic geometric signatures.

\subsection{Discrete Sectional Curvature}
\label{app:sectional-curvature}

Let $G=(V,E)$ be an undirected graph with adjacency matrix $A$ and pairwise distance matrix $D$, where $D_{uv}=d_G(u,v)$ denotes the graph distance between nodes $u$ and $v$. In our implementation, $D$ is typically instantiated as the all-pairs shortest-path distance matrix. For disconnected graphs, invalid or infinite distances are excluded from the corresponding averaging operations.

For a center node $m\in V$, we consider unordered neighbor pairs
\[
    \mathcal{P}_m
    =
    \bigl\{\{b,c\}: b,c\in \mathcal{N}(m),\ b<c\bigr\},
\]
where $\mathcal{N}(m)=\{v\in V:(m,v)\in E\}$ is the one-hop neighborhood of $m$. Each pair $\{b,c\}$ defines a local metric section around $m$. To probe this section, we further introduce an anchor node $a\in V\setminus\{m\}$ satisfying $0<d_G(a,m)<\infty$. For every valid quadruple $(a,b,c;m)$, we define the midpoint curvature residual as
\[
    \Delta_G(a,b,c;m)
    =
    d_G(a,m)^2
    +
    \frac{1}{4}d_G(b,c)^2
    -
    \frac{1}{2}
    \Bigl(
        d_G(a,b)^2+d_G(a,c)^2
    \Bigr).
\]
The normalized sectional curvature residual is then given by
\[
    \xi_G(a,b,c;m)
    =
    \frac{
    d_G(a,m)^2
    +
    \frac{1}{4}d_G(b,c)^2
    -
    \frac{1}{2}
    \left(
        d_G(a,b)^2+d_G(a,c)^2
    \right)}
    {2d_G(a,m)}.
    \label{eq:app-xi}
\]
This formula follows the Euclidean midpoint identity. If the local metric geometry around $m$ behaves approximately like a flat Euclidean section, then the residual in Eq.~\eqref{eq:app-xi} is close to zero. Positive values indicate locally ``fatter''-than-Euclidean geometry, while negative values indicate locally ``thinner''-than-Euclidean geometry, typically associated with geodesic divergence or tree-like expansion.

The raw curvature estimate of node $m$ is obtained by averaging the residual over all valid anchor nodes and all unordered neighbor pairs:
\[
    \widehat{\kappa}_G(m)
    =
    \frac{1}{|\mathcal{P}_m|}
    \sum_{\{b,c\}\in \mathcal{P}_m}
    \frac{1}{|\mathcal{A}_m|}
    \sum_{a\in \mathcal{A}_m}
    \xi_G(a,b,c;m),
    \label{eq:app-node-curv}
\]
where
\[
    \mathcal{A}_m
    =
    \{a\in V\setminus\{m\}:0<d_G(a,m)<\infty\}.
\]
When $|\mathcal{P}_m|=0$, i.e., when the degree of $m$ is smaller than two, we set $\widehat{\kappa}_G(m)=0$. This convention avoids introducing artificial curvature signals for nodes whose local metric section cannot be formed.

For cross-dataset comparability, we further use a relative curvature normalization:
\[
    \kappa_G(m)
    =
    \frac{\widehat{\kappa}_G(m)}
    {\operatorname{diam}_f(G)},
    \qquad
    \operatorname{diam}_f(G)
    =
    \max_{u,v:\ d_G(u,v)<\infty} d_G(u,v).
    \label{eq:app-relative-curv}
\]
This normalization removes the global scale effect caused by differences in graph diameter. Without this step, datasets with larger graph distances may exhibit curvature magnitudes that are not directly comparable to smaller graphs, even when their local geometric patterns are similar.

\paragraph{Implementation details.}
Our curvature computation is implemented as a GPU-accelerated tensor routine. Given an adjacency matrix $A$ and a corresponding distance matrix $D$, the algorithm iterates over center nodes $m$ to extract neighbor sets, forms unordered pairs $\{b,c\}$, and evaluates Eq.~\eqref{eq:app-xi} across all valid anchors $a$. To ensure scalability and reliability, the implementation offers two operational modes: 

\texttt{Fast} Mode: Utilizes lightweight, chunked computation designed for medium-scale graphs, prioritizing throughput.

\texttt{Strict} Mode: Employs conservative memory access and accumulates statistics in higher precision, making it the preferred choice for large-scale experiments and final reporting.

To manage memory overhead, both modes utilize pair-wise chunking to avoid the simultaneous materialization of all anchor-pair combinations. Furthermore, the pipeline automatically filters invalid denominators and infinite distances. Once node-level curvature values are computed, we apply relative normalization based on the maximum finite entry in the distance matrix to ensure stability across different graph scales.

\subsection{Curvature Skewness}
\label{app:curvature-skewness}

The mean curvature of a graph provides a first-order summary of its global geometric tendency:
\[
    \overline{\kappa}(G)
    =
    \frac{1}{|V|}
    \sum_{m\in V}\kappa_G(m).
    \label{eq:app-mean-curv}
\]
However, mean curvature alone is insufficient for characterizing heterogeneous relational data. A graph may have near-zero average curvature while still containing a small subset of nodes with strongly positive or strongly negative curvature. Such tail behavior is especially important in table-derived graphs, where foreign-key joins, hub entities, and relational schema construction can induce highly asymmetric local geometry. Therefore, CURVBENCH also computes the skewness of the node-level curvature distribution.

Let
\[
    \sigma_\kappa(G)
    =
    \left(
    \frac{1}{|V|}
    \sum_{m\in V}
    \left(\kappa_G(m)-\overline{\kappa}(G)\right)^2
    \right)^{1/2}
\]
be the standard deviation of node-level curvature values. The curvature skewness is defined as the third standardized central moment:
\[
    \gamma_\kappa(G)
    =
    \begin{cases}
    \displaystyle
    \frac{1}{|V|}
    \sum_{m\in V}
    \left(
    \frac{\kappa_G(m)-\overline{\kappa}(G)}
    {\sigma_\kappa(G)}
    \right)^3,
    & \sigma_\kappa(G)>0,\\[2.2ex]
    0,
    & \sigma_\kappa(G)=0.
    \end{cases}
    \label{eq:app-skewness}
\]
Positive skewness indicates a right-tailed curvature distribution, where a minority of nodes exhibit substantially higher positive curvature than the graph average. Negative skewness indicates a left-tailed distribution, where a minority of nodes exhibit substantially lower curvature. Values close to zero suggest a comparatively balanced curvature profile.

\paragraph{Implementation details.}
In the implementation, node-level curvature tensors are first loaded from saved \texttt{.pt} files. Non-finite entries are removed before computing statistics. The skewness is then computed directly as
\[
    \operatorname{mean}
    \left[
        \left(
        \frac{\kappa-\mu}{\sigma}
        \right)^3
    \right],
\]
where $\mu$ and $\sigma$ are the empirical mean and standard deviation of finite node curvature values. 


\subsection{Why Both Metrics Are Needed}
\label{app:why-two-metrics}

The synergy between mean curvature $\overline{\kappa}(G)$ and curvature skewness $\gamma_\kappa(G)$ is fundamental to \textsc{CurvBench}. While mean curvature captures a dataset's average signed geometry, it often masks local heterogeneity. Skewness complements this by revealing whether the curvature distribution is dominated by asymmetric tails—a distinction critical to relational learning, where performance is often driven by rare but influential geometric regions rather than just the global average.

\paragraph{Geometric Nuance and Inductive Bias.} This dual-metric approach prevents the misclassification of complex topologies: 

Positive Skewness ($\gamma_\kappa(G) \gg 0$): A graph with $\overline{\kappa}(G) \approx 0$ may appear Euclidean, yet a high positive skew indicates a strong positive-curvature tail. Such structures favor models designed for compact, clustered, or schema-induced patterns.

Negative Skewness ($\gamma_\kappa(G) < 0$): Conversely, negative skewness points to tree-like or hierarchically expanding regions, which align more effectively with hyperbolic or mixed-curvature inductive biases. 

Relying solely on $\overline{\kappa}(G)$ would collapse these distinct geometries into a single regime, inadvertently reproducing the same failure modes found in conventional aggregate benchmarks.

\paragraph{The Role of Metrics in \textsc{CurvBench}.} Within the \textsc{CurvBench} framework, these metrics fulfill two primary functions: 

Interpretability: They provide a pre-training geometric summary of each dataset.

Stratification: They facilitate regime-stratified evaluation, conditioning model comparisons on the data's geometric profile.

By moving away from arbitrary dataset mixtures, \textsc{CurvBench} shifts the focus from identifying a "globally best" model to determining which inductive biases are most effective under specific geometric conditions.

\subsection{Ranking consistency metrics}

To quantify whether datasets within the same geometry regime induce similar model preferences, we compare the model rankings produced by different datasets using Spearman correlation, Kendall rank correlation, and top-\(k\) Jaccard overlap.

Let \(\mathcal M=\{M_1,\ldots,M_N\}\) denote the set of evaluated models. 
For a dataset \(G\), let \(r_G(M_i)\) be the rank position of model \(M_i\), where a smaller value indicates better performance. 
In our top-3 truncated ranking protocol, the top three models retain their exact ranks, while all remaining models are assigned the same lower-tier rank. 
This focuses the comparison on the most competitive models while treating lower-ranked methods as indistinguishable.

\textbf{Spearman correlation.}
Spearman correlation measures the global monotonic agreement between two rankings. 
Given two datasets \(G\) and \(G'\), we compute Spearman correlation as the Pearson correlation between their rank vectors:
\begin{equation}
\label{eq:spearman}
\rho(G,G')
=
\frac{
    \sum_{i=1}^{N}
    \left(r_G(M_i)-\bar r_G\right)
    \left(r_{G'}(M_i)-\bar r_{G'}\right)
}{
    \sqrt{
    \sum_{i=1}^{N}
    \left(r_G(M_i)-\bar r_G\right)^2
    }
    \sqrt{
    \sum_{i=1}^{N}
    \left(r_{G'}(M_i)-\bar r_{G'}\right)^2
    }
},
\end{equation}
where
\[
    \bar r_G=\frac{1}{N}\sum_{i=1}^{N}r_G(M_i),
    \qquad
    \bar r_{G'}=\frac{1}{N}\sum_{i=1}^{N}r_{G'}(M_i).
\]
A larger \(\rho(G,G')\) indicates that models ranked highly on \(G\) also tend to be ranked highly on \(G'\). 
Thus, Spearman correlation captures global consistency over the whole ranked model list.

\textbf{Kendall rank correlation.}
Kendall correlation measures pairwise ordering agreement between two rankings. 
For each model pair \((M_i,M_j)\), we compare whether the relative order of the pair is preserved across two datasets. 
Define
\[
s_G(i,j)
=
\operatorname{sign}\left(r_G(M_j)-r_G(M_i)\right),
\]
where \(s_G(i,j)=+1\) means \(M_i\) ranks above \(M_j\), \(s_G(i,j)=-1\) means \(M_j\) ranks above \(M_i\), and \(s_G(i,j)=0\) means the two models are tied. 
The Kendall correlation can then be written as
\begin{equation}
\label{eq:kendall}
\tau(G,G')
=
\frac{
    \sum_{i<j}
    s_G(i,j)s_{G'}(i,j)
}{
    \sqrt{
    \sum_{i<j}s_G(i,j)^2
    }
    \sqrt{
    \sum_{i<j}s_{G'}(i,j)^2
    }
}.
\end{equation}
This formulation naturally handles ties induced by the top-3 truncation. 
A high Kendall correlation means that pairwise model preferences are stable: if \(M_i\) outperforms \(M_j\) on one dataset, the same ordering is likely to hold on the other dataset.

\textbf{Top-\(k\) Jaccard overlap.}
While Spearman and Kendall correlations compare ranked lists, Jaccard overlap compares the identity of the top-performing models. 
Let
\[
    T_k(G)=\{M_i: M_i \text{ is among the top-}k \text{ models on } G\}.
\]
The top-\(k\) Jaccard overlap between two datasets is defined as
\begin{equation}
\label{eq:jaccard}
J_k(G,G')
=
\frac{
    |T_k(G)\cap T_k(G')|
}{
    |T_k(G)\cup T_k(G')|
}.
\end{equation}
In our experiments, we set \(k=3\). 
Thus, \(J_3(G,G')\) measures whether two datasets share the same leading models, regardless of the exact ordering among lower-ranked methods.

Together, these metrics evaluate ranking consistency at different levels. 
Spearman correlation measures global monotonic agreement, Kendall correlation measures pairwise preference stability, and top-\(3\) Jaccard overlap measures agreement among the strongest models. 
If within-regime similarity is consistently higher than cross-regime similarity under all three metrics, this indicates that datasets in the same curvature regime induce coherent model preferences.

\section{Related Work}
\label{app:related-work}

\paragraph{Relational Learning.}
Relational learning aims to model data instances connected through explicit or implicit dependencies, covering a broad range of domains such as citation networks, molecular graphs, social systems, knowledge graphs, and table-derived relational data~\cite{battaglia2018relational, fey2023relational, sun2026survey}. Although these datasets differ substantially in scale, sparsity, feature modality, and task formulation, they share a common structural abstraction: relations induce paths, neighborhoods, and higher-order connectivity patterns that define an underlying topology~\cite{eremeev2025turning, du2026graphoracle}. This has motivated the development of unified relational learning benchmarks, where diverse datasets are collected to evaluate model generalization across domains. However, most existing benchmarks summarize model performance through flat leaderboards that average results across heterogeneous datasets~\cite{lv2021we}. Such an evaluation protocol implicitly assumes that relational data forms a structurally uniform category, while in practice different datasets may exhibit fundamentally different geometric properties. As a result, aggregate rankings can obscure regime-specific strengths and weaknesses, making it difficult to determine whether a model is genuinely robust or merely well aligned with the dominant structure of a benchmark.

\paragraph{Benchmarks for relational learning.}
Standardized benchmarks—most notably OGB, Benchmarking GNNs, and GraphGym—have served as the bedrock of relational learning by formalizing reproducibility through common datasets and evaluation protocols \citep{hu2020open,dwivedi2023benchmarking,you2020design}. Recent expansions have even extended these frameworks to relational databases and table-derived graphs, mapping rows and foreign-key relations into structured graph formats \citep{fey2023relational,kim2024carte}. However, a fundamental limitation persists: most benchmarks rely on flat leaderboards that aggregate performance across heterogeneous datasets. While useful for coarse comparisons, these global averages implicitly assume a degree of structural uniformity that rarely exists in practice. Datasets vary not only in scale and homophily but also in intrinsic metric geometry. Consequently, aggregate rankings often obscure regime-dependent behaviors, making model superiority appear more universal than the underlying data geometry justifies. \textsc{CurvBench} addresses this by introducing curvature as an explicit evaluation axis, shifting the paradigm from global rankings to regime-conditioned analysis.

\paragraph{Non-Euclidean Graph Learning and Geometry-Aware Evaluation.}
Research has increasingly demonstrated that graph representation learning is heavily influenced by the geometry of the embedding space. While Euclidean-based models like GCN, GAT, and GraphSAGE remain robust baselines \citep{kipf2017semi,velickovic2018graph,hamilton2017inductive}, Euclidean space is often fundamentally misaligned with hierarchical, tree-like, or mixed-structure graphs. This has motivated the rise of hyperbolic neural networks and embeddings designed for negatively curved spaces \citep{nickel2017poincare,ganea2018hyperbolic}, resulting in specialized architectures such as HGCN and hyperbolic attention mechanisms \citep{chami2019hyperbolic,zhang2021hyperbolic}.

To better match non-Euclidean structures~\cite{asif2021graph}, a growing body of work has introduced multiple experts into graph learning. These methods explicitly encode geometric inductive biases, aiming to better capture hierarchical, clustered, or heterogeneous graph structures~\cite{ ye2025mose, wang2025cooperation}. More recent advances in pseudo-Riemannian, mixed-curvature, and adaptive Riemannian models further relax the assumption of a static geometry \citep{xiong2022pseudo,guo2024graphmore}. Despite being predicated on the idea that representation geometry should match data geometry, these models are still typically assessed using the same flat averages as their Euclidean counterparts. \textsc{CurvBench} closes this gap by stratifying datasets via mean curvature and curvature skewness. This enables a rigorous assessment of when Euclidean, hyperbolic, adaptive, or foundation models are truly aligned with the underlying relational geometry.

\section{Additional Experiments and Observations}
\label{app:additional-analysis}

This section provides supplementary analyses that build upon our primary findings. While the main text focuses on curvature-conditioned rankings and the reorganization of model families, we further examine four critical diagnostic dimensions:
(i) Leaderboard Deviation: We quantify the degree to which each curvature-stratified leaderboard diverges from the conventional flat leaderboard.
(ii) Task Consistency: We investigate whether curvature-conditioned preferences remain stable across diverse relational tasks. 
(iii) GFM Supervision and Feasibility: We analyze the interaction between additional supervision signals and the practical feasibility of deploying Graph Foundation Models.
(iv) Specialist–Robustness Trade-offs: We evaluate whether table-derived graphs necessitate a trade-off between regime-specific expertise and general model robustness.

\begin{table}[htbp]
\centering
\caption{Macro-F1 results on Node Classification (NC) task. Datasets and baselines are divided into different regimes. We highlight the top-3 results with \textbf{\textcolor{red!80!black}{red bolded}}, \textcolor{red}{red} and \textbf{bold}.}
\label{tab:app-nc-f1-results}
\resizebox{\linewidth}{!}{
\begin{tabular}{c|ccc|ccc|ccc}
\toprule
Baselines & Cora         & Citeseer     & PubMed       & Airport      & Cornell      & Actor        & Disease      & Telecom      & CS$\_$Phds      \\
\midrule
GCN       & 79.36$\pm$1.13   & 65.64$\pm$0.66   & 77.90$\pm$0.29   & 66.10$\pm$1.16   & 24.88$\pm$3.07   & 22.74$\pm$1.02   & 59.19$\pm$2.59   & 57.19$\pm$0.46   & 26.47$\pm$0.84  \\
GAT       & 80.02$\pm$0.45   & 64.43$\pm$1.03   & 76.69$\pm$0.25   & \textbf{82.05$\pm$1.18}   & 27.71$\pm$2.35    & 22.63$\pm$1.20   & \textbf{84.41$\pm$2.86}   & 52.51$\pm$0.15   & 10.58$\pm$0.00   \\
GraphSAGE & \textbf{\textcolor{red!80!black}{87.60$\pm$0.28}}   & \textcolor{red}{71.59$\pm$0.53} & \textcolor{red}{88.09$\pm$0.06}   & 29.24$\pm$0.56   & \textbf{\textcolor{red!80!black}{51.72$\pm$5.97}} & \textbf{31.82$\pm$0.53} & \textcolor{red}{91.50$\pm$3.04} & \textbf{\textcolor{red!80!black}{71.84$\pm$0.38}} & 10.47$\pm$1.95   \\
MLP       & 54.86$\pm$1.06   & 54.18$\pm$0.87   & 71.41$\pm$0.71   & \textcolor{red}{85.90$\pm$0.81}   & \textbf{51.25$\pm$1.71} & \textbf{\textcolor{red!80!black}{35.46$\pm$1.90}}   & 44.41$\pm$0.00   & 58.45$\pm$0.03   & 10.58$\pm$0.00   \\
PCNet     & \textcolor{red}{87.41$\pm$0.47} & \textbf{\textcolor{red!80!black}{73.15$\pm$0.15}} & \textbf{\textcolor{red!80!black}{89.30$\pm$0.14}} & 30.43$\pm$0.50 & 48.42$\pm$5.91 & 31.62$\pm$1.00 & 49.78$\pm$1.23 & 57.97$\pm$0.03 & 12.62$\pm$1.08   \\
\midrule
HGNN      & 57.51$\pm$0.38 & 58.09$\pm$0.42 & 73.02$\pm$0.66 & 24.56$\pm$3.17 & \textcolor{red}{51.54$\pm$5.24} & \textcolor{red}{33.72$\pm$0.60} & 39.16$\pm$2.92 & 58.18$\pm$0.04 & 18.84$\pm$2.51 \\
HAT       & \textbf{80.33$\pm$0.31} & \textbf{67.79$\pm$0.32}   & \textbf{78.27$\pm$0.42}   & 58.55$\pm$6.60   & 11.34$\pm$0.13  & 26.36$\pm$0.52   & 43.67$\pm$0.00   & 58.24$\pm$0.01   & 10.57$\pm$0.00  \\
HGCNN     & 77.37$\pm$0.73 & 65.03$\pm$0.43 & 76.33$\pm$0.47 & 43.60$\pm$5.96 & 45.55$\pm$2.35 & 21.07$\pm$1.34 & 76.88$\pm$2.35 & \textbf{61.96$\pm$8.29} & \textcolor{red}{42.57$\pm$2.98} \\
HyboNet & 73.83$\pm$0.39   & 65.18$\pm$1.53   & 73.46$\pm$0.47   & 47.06$\pm$3.93   & 18.48$\pm$5.23   & 17.81$\pm$2.55   & 76.47$\pm$5.23   & 38.45$\pm$11.22  & 10.54$\pm$0.06  \\
\midrule
CUSP      & 76.12$\pm$1.17  & 64.22$\pm$1.00  & 65.99$\pm$2.16  & 47.86$\pm$2.71  & 13.53$\pm$0.00  & 13.33$\pm$3.16  & 43.74$\pm$19.17 & 36.18$\pm$6.88  & 12.77$\pm$3.52  \\
QGCN      & 78.53$\pm$0.48 & 63.87$\pm$0.19 & 75.66$\pm$0.98 & 47.17$\pm$0.79 & 26.81$\pm$1.54 & 21.61$\pm$0.76 & 83.01$\pm$1.52 & \textcolor{red}{66.46$\pm$0.27} & \textbf{\textcolor{red!80!black}{43.86$\pm$2.53}} \\
\midrule
GraphMoRE & 80.21$\pm$0.23 & 64.59$\pm$1.03 & 76.01$\pm$1.02 & \textbf{\textcolor{red!80!black}{90.47$\pm$1.19}} & 19.62$\pm$2.31 & 22.76$\pm$0.66 & \textbf{\textcolor{red!80!black}{93.80$\pm$1.34}} & 64.49$\pm$0.74 & \textbf{36.12$\pm$2.94} \\
\bottomrule
\end{tabular}
}
\end{table}

\begin{table}[htbp]
\centering
\caption{AP results on Link Prediction (LP) task. Datasets and baselines are divided into different regimes. We highlight the top-3 results with \textbf{\textcolor{red!80!black}{red bolded}}, \textcolor{red}{red} and \textbf{bold}.}
\label{tab:app-lp-ap-results}
\resizebox{\linewidth}{!}{
\begin{tabular}{c|ccc|ccc|ccc}
\toprule
Baselines & Cora & Citeseer & PubMed & Airport & Cornell & Actor & Disease & Telecom & CS\_Phds \\
\midrule
GCN       & \textbf{92.18$\pm$1.18} & \textbf{93.16$\pm$0.56} & 92.28$\pm$0.26 & 94.04$\pm$0.66 & \textcolor{red}{73.72$\pm$10.41} & \textbf{82.59$\pm$0.63} & 49.58$\pm$2.75 & \textcolor{red}{69.37$\pm$0.77} & 42.14$\pm$0.48 \\
GAT       & \textcolor{red}{92.46$\pm$0.28} & \textcolor{red}{93.58$\pm$0.52} & 91.02$\pm$0.22 & 93.20$\pm$0.65 & \textbf{\textcolor{red!80!black}{73.80$\pm$6.49}} & 81.88$\pm$0.98 & 50.13$\pm$3.01 & \textbf{\textcolor{red!80!black}{69.79$\pm$0.86}} & 47.56$\pm$0.28 \\
GraphSAGE & 68.19$\pm$1.10 & 68.52$\pm$1.47 & 84.31$\pm$0.64 & 71.33$\pm$0.84 & 56.01$\pm$6.86 & 56.24$\pm$1.29 & 49.98$\pm$0.05 & 59.68$\pm$0.52 & 43.34$\pm$0.17 \\
MLP       & 82.25$\pm$1.46 & 89.59$\pm$1.57 & 85.29$\pm$0.50 & 90.71$\pm$0.67 & 67.46$\pm$6.49 & 70.85$\pm$0.93 & 51.50$\pm$0.35 & 67.60$\pm$0.91 & 50.00$\pm$0.00 \\
PCNet     & 73.20$\pm$0.52 & 68.87$\pm$2.03 & \textcolor{red}{94.05$\pm$0.13} & 71.98$\pm$0.73 & 56.44$\pm$5.16 & 63.74$\pm$1.04 & 58.12$\pm$8.37 & 67.22$\pm$0.42 & \textbf{\textcolor{red!80!black}{83.49$\pm$0.76}} \\
\midrule
HGNN      & 68.96$\pm$0.95 & 84.47$\pm$0.48 & 90.83$\pm$0.24 & 93.48$\pm$0.24 & 64.94$\pm$4.90 & 71.46$\pm$0.55 & 51.90$\pm$1.80 & \textbf{68.81$\pm$0.16} & 50.84$\pm$2.44 \\
HGCNN     & 83.54$\pm$1.81 & 89.06$\pm$0.62 & \textbf{94.01$\pm$0.07} & 93.68$\pm$0.21 & \textbf{71.16$\pm$3.02} & 82.19$\pm$0.31 & \textbf{\textcolor{red!80!black}{66.54$\pm$6.12}} & 59.10$\pm$1.28 & 55.19$\pm$1.72 \\
Fully-HNN & 90.33$\pm$1.65 & 79.27$\pm$1.46 & 92.20$\pm$0.41 & \textbf{\textcolor{red!80!black}{96.16$\pm$0.70}} & 68.27$\pm$4.79 & \textcolor{red}{85.14$\pm$1.05} & 46.55$\pm$1.50 & 55.18$\pm$0.72 & \textcolor{red}{59.38$\pm$4.89} \\
\midrule
CUSP      & 88.07$\pm$1.55 & 90.12$\pm$1.86 & 60.04$\pm$0.90 & 73.24$\pm$0.97 & 61.89$\pm$9.41 & 72.81$\pm$1.45 & 36.87$\pm$1.85 & 64.71$\pm$0.90 & \textbf{56.29$\pm$2.18} \\
GraphMoRE & \textbf{\textcolor{red!80!black}{96.87$\pm$0.15}} & \textbf{\textcolor{red!80!black}{98.48$\pm$0.28}} & \textbf{\textcolor{red!80!black}{98.61$\pm$0.18}} & \textbf{94.73$\pm$0.42} & 64.71$\pm$3.16 & \textbf{\textcolor{red!80!black}{87.55$\pm$0.57}} & \textcolor{red}{65.21$\pm$2.31} & 68.30$\pm$0.54 & 46.86$\pm$3.16 \\
QGCN      & 88.35$\pm$0.26 & 88.85$\pm$0.63 & 93.40$\pm$0.23 & \textcolor{red}{95.69$\pm$0.08} & 62.61$\pm$1.66 & 79.70$\pm$0.67 & \textbf{64.29$\pm$1.20} & 64.76$\pm$0.69 & 54.98$\pm$1.71 \\
\bottomrule
\end{tabular}
}
\end{table}

\paragraph{Rank-shift analysis across tasks and metrics.}
Using Tables~\ref{tab:app-nc-acc-results}, \ref{tab:app-nc-f1-results}, 
\ref{tab:app-lp-auc-results}, and~\ref{tab:app-lp-ap-results}, we compute the RSI for each curvature regime under four evaluation views: NC performance, NC Macro-F1, LP AUC, and LP AP. 
This multi-metric analysis avoids drawing conclusions from a single task-metric pair and allows us to examine whether the distortion induced by flat averaging is stable across evaluation criteria.

Table~\ref{tab:multi-metric-rsi} shows that rank distortion is not uniform across tasks. 
For node classification, both performance and Macro-F1 lead to the same conclusion: flat averaging is most misleading in the near-zero and positive regimes. 
The average RSI is $2.92$ in the near-zero regime and $2.50$ in the positive regime, but only $1.25$ in the negative regime. 
This indicates that, for NC, the global leaderboard hides substantial regime-specific reorganization among citation-like and positively curved graphs. 
For example, Tables~\ref{tab:app-nc-acc-results} and~\ref{tab:app-nc-f1-results} consistently show that PCNet and GraphSAGE dominate the near-zero regime, while MLP becomes highly competitive in the positive regime, especially under Macro-F1. 
Such behavior would be obscured if all datasets were collapsed into a single flat average.

\begin{wraptable}{r}{0.55\linewidth}
\centering
\caption{Rank-shift index (RSI) between flat and regime-conditioned rankings across tasks and metrics. 
Larger values indicate stronger deviation from the flat leaderboard.}
\label{tab:multi-metric-rsi}
\resizebox{0.72\linewidth}{!}{
\begin{tabular}{lccc}
\toprule
Evaluation view & Near-zero & Positive & Negative \\
\midrule
NC Performance & 3.00 & 2.50 & 1.17 \\
NC Macro-F1    & 2.83 & 2.50 & 1.33 \\
LP AUC         & 0.91 & 1.45 & 2.18 \\
LP AP          & 0.73 & 1.82 & 2.55 \\
\midrule
Task-wise NC average & 2.92 & 2.50 & 1.25 \\
Task-wise LP average & 0.82 & 1.64 & 2.36 \\
\bottomrule
\end{tabular}
}
\end{wraptable}

The LP task exhibits the opposite pattern. 
Across both AUC and AP, the largest distortion appears in the negative regime, with an average RSI of $2.36$, compared with $0.82$ in the near-zero regime. 
This suggests that LP rankings are relatively stable on citation-like graphs, but become much more geometry-sensitive on negatively curved graphs. 
Tables~\ref{tab:app-lp-auc-results} and~\ref{tab:app-lp-ap-results} show that PCNet and QGCN become particularly strong in the negative regime, whereas GraphMoRE dominates the near-zero regime under both LP metrics. 
Therefore, the same flat leaderboard can fail in different ways depending on the downstream task: NC aggregation mainly hides near-zero and positive-regime shifts, while LP aggregation mainly hides negative-regime specialization.


\paragraph{Observation G.1: Flat leaderboards hide task-dependent rank distortions.}
The multi-metric RSI results in Table~\ref{tab:multi-metric-rsi} strengthen the central claim of CURVBENCH. 
Flat leaderboards do not merely introduce random noise; they induce systematic and task-dependent distortions. 
For NC, the main distortion comes from feature- and classification-sensitive regimes, whereas for LP, the main distortion comes from negatively curved regimes where pairwise proximity and structural closure become more important. 
Thus, regime-conditioned evaluation is not only a more fine-grained reporting format, but changes the substantive conclusion about which inductive bias is reliable under which structural condition.

\begin{wraptable}{r}{0.55\linewidth}
\centering
\caption{Within-task consistency between two metrics under each curvature regime. 
NC compares Accuracy and Macro-F1; LP compares AUC and AP.}
\label{tab:within-task-metric-consistency}
\resizebox{0.72\linewidth}{!}{
\begin{tabular}{llccc}
\toprule
Task & Regime & Spearman & Kendall & Top-3 Jaccard \\
\midrule
NC & Near-zero & 0.937 & 0.848 & 1.000 \\
NC & Positive  & 0.769 & 0.606 & 0.500 \\
NC & Negative  & 0.860 & 0.758 & 0.500 \\
\midrule
LP & Near-zero & 0.991 & 0.964 & 1.000 \\
LP & Positive  & 0.836 & 0.673 & 0.500 \\
LP & Negative  & 0.864 & 0.709 & 0.500 \\
\bottomrule
\end{tabular}
}
\end{wraptable}

\paragraph{Within-task cross-metric consistency.}
We next examine whether different metrics within the same task induce similar model preferences. 
For NC, we compare the regime-wise rankings produced by Tables~\ref{tab:app-nc-acc-results} and~\ref{tab:app-nc-f1-results}. 
For LP, we compare the regime-wise rankings produced by Tables~\ref{tab:app-lp-auc-results} and~\ref{tab:app-lp-ap-results}. 
The results are shown in Table~\ref{tab:within-task-metric-consistency}.

Table~\ref{tab:within-task-metric-consistency} shows that metric choice does not destroy the curvature-conditioned structure of the results. 
Within each task, the two metrics produce highly consistent rankings in the near-zero regime: NC achieves a Spearman correlation of $0.937$ and LP achieves $0.991$, with identical top-3 model sets in both cases. 
This indicates that citation-like graphs induce stable model preferences regardless of the precise metric used. 
However, consistency becomes weaker in the positive and negative regimes, where the top-3 Jaccard overlap drops to $0.5$. 
This suggests that non-flat regimes are more sensitive not only to model geometry, but also to what aspect of performance is measured. 
For example, in LP, AUC and AP agree that the near-zero regime strongly favors GraphMoRE, GCN, and GAT, but differ more noticeably in the positive and negative regimes.

\paragraph{Cross-task order consistency.}
Finally, we analyze whether NC and LP induce similar model preferences under the same curvature regime. 
Instead of comparing only one NC metric against one LP metric, we average the ranking-consistency scores over all four NC--LP metric pairs, where results are summarized in Table~\ref{tab:cross-task-consistency-multimetric}. 
For rank-based statistics, we use the common model set shared by the corresponding NC and LP tables.

\begin{wraptable}{r}{0.48\linewidth}
\vspace{-1em}
\centering
\caption{Cross-task consistency between NC and LP rankings, averaged over all NC--LP metric pairs. 
Higher values indicate stronger agreement between task-induced model orders.}
\label{tab:cross-task-consistency-multimetric}
\resizebox{\linewidth}{!}{
\begin{tabular}{lccc}
\toprule
Regime & Spearman & Kendall & Top-3 Jaccard \\
\midrule
Near-zero & 0.011 & 0.127 & 0.200 \\
Positive  & 0.066 & 0.036 & 0.275 \\
Negative  & 0.352 & 0.255 & 0.425 \\
\bottomrule
\end{tabular}
}
\vspace{-1em}
\end{wraptable}

Table~\ref{tab:cross-task-consistency-multimetric} shows that cross-task agreement is much weaker than within-task cross-metric agreement. 
In the near-zero regime, NC and LP have almost no global rank correlation, even though both tasks are evaluated on the same citation-like datasets. 
Tables~\ref{tab:app-nc-acc-results} and~\ref{tab:app-nc-f1-results} show that NC favors PCNet and GraphSAGE, whereas Tables~\ref{tab:app-lp-auc-results} and~\ref{tab:app-lp-ap-results} show that LP favors GraphMoRE together with strong Euclidean baselines such as GCN and GAT. 
This indicates that near-zero geometry alone does not determine a universal model order; the task objective determines which structural signal becomes useful.

The negative regime exhibits the strongest cross-task agreement, with Spearman correlation increasing to $0.352$ and top-3 Jaccard overlap to $0.425$. 
This is because both NC and LP repeatedly favor geometry-aware or structure-sensitive methods in negatively curved graphs, especially GraphMoRE, QGCN, HGCNN, and PCNet, although their exact ordering still differs by task and metric. 
The positive regime remains intermediate and metric-sensitive: MLP is highly competitive for NC, while LP favors methods such as GAT, GraphMoRE, and the hyperbolic baseline. 
Thus, curvature provides a meaningful structural context, but the final ranking is determined by the interaction among geometry, task objective, and evaluation metric.

\paragraph{Observation G.2: Model effectiveness is jointly shaped by geometry, task, and metric.}
The combined analysis of Tables~\ref{tab:multi-metric-rsi}, \ref{tab:within-task-metric-consistency}, and~\ref{tab:cross-task-consistency-multimetric} refines the interpretation of CURVBENCH. 
Curvature is not a standalone oracle that assigns one universal ranking to each regime. 
Instead, it defines the structural context in which task-specific and metric-specific preferences emerge. 
Within a fixed task, different metrics generally preserve the same broad curvature-conditioned trends, especially in the near-zero regime. 
Across tasks, however, rankings can diverge substantially because NC emphasizes feature separability, neighborhood aggregation, and label smoothness, whereas LP emphasizes pairwise proximity, structural closure, and long-range connectivity. 
Therefore, model effectiveness is jointly shaped by intrinsic geometry, downstream objective, and evaluation metric. 
This is precisely why CURVBENCH replaces a single flat leaderboard with regime-conditioned diagnostics: it reveals not only which model performs well, but also under which geometric and task-metric conditions its inductive bias becomes effective.

\begin{table}[h]
\centering
\caption{Macro-F1 results of GFMs under 1-shot and 5-shot scenarios. OOM means Out-Of-Memory.}
\label{tab:app-gfm-fewshot}
\resizebox{\linewidth}{!}{
\begin{tabular}{c|ccc|ccc|ccc}
\toprule
\multicolumn{10}{c}{\textbf{1-shot scenario}} \\
\midrule
Baselines   & Cora          & Citeseer      & PubMed        & Airport       & Cornell       & Actor        & Disease       & Telecom       & CS$\_$Phds      \\
\midrule
GCOPE       & 30.57$\pm$6.40  & 29.65$\pm$10.77 & 33.28$\pm$7.39  & \textbf{\textcolor{red!80!black}{18.67$\pm$7.17}} & \textbf{24.03$\pm$6.35} & 14.55$\pm$5.14 & \textbf{43.21$\pm$2.38} & \textbf{\textcolor{red!80!black}{41.49$\pm$6.56}} & 22.07$\pm$4.64 \\
MDGPT       & \textbf{\textcolor{red!80!black}{42.75$\pm$8.24}} & \textcolor{red}{34.86$\pm$9.80} & \textbf{\textcolor{red!80!black}{51.89$\pm$11.34}} & 6.65$\pm$5.31 & \textbf{28.14$\pm$7.47} & 15.92$\pm$3.44 & \textcolor{red}{43.85$\pm$9.81} & \textbf{33.42$\pm$8.98} & 22.78$\pm$2.57 \\
MDGFM       & \textcolor{red}{42.53$\pm$5.32} & \textbf{\textcolor{red!80!black}{37.66$\pm$6.14}} & \textcolor{red}{50.66$\pm$10.12} & 14.51$\pm$2.83 & \textcolor{red}{29.28$\pm$6.17} & \textcolor{red}{17.74$\pm$2.12} & 37.44$\pm$4.52 & OOM & \textcolor{red}{23.62$\pm$2.23} \\
SAMGPT      & 25.01$\pm$8.42 & \textbf{33.37$\pm$7.73} & 39.46$\pm$9.91 & 13.35$\pm$6.17 & \textbf{\textcolor{red!80!black}{30.05$\pm$7.22}} & \textbf{\textcolor{red!80!black}{17.92$\pm$5.71}} & \textbf{43.70$\pm$5.65} & \textcolor{red}{34.74$\pm$8.46} & \textbf{\textcolor{red!80!black}{23.77$\pm$6.50}} \\
GraphGluing & 12.98$\pm$6.47 & 18.03$\pm$9.63 & 33.07$\pm$5.72 & \textcolor{red}{18.34$\pm$2.79} & 19.21$\pm$4.26 & 8.28$\pm$0.16 & \textbf{\textcolor{red!80!black}{47.39$\pm$4.25}} & OOM & 19.42$\pm$6.71 \\
SA2GFM      & \textbf{37.97$\pm$8.36} & 26.45$\pm$6.44 & \textbf{42.33$\pm$9.73} & \textbf{17.23$\pm$3.96} & 18.50$\pm$3.40 & \textbf{15.94$\pm$2.06} & 42.67$\pm$6.90 & OOM & \textbf{23.16$\pm$3.89} \\
\midrule
\multicolumn{10}{c}{\textbf{5-shot scenario}} \\
\midrule
Baselines   & Cora         & Citeseer     & PubMed       & Airport       & Cornell       & Actor        & Disease       & Telecom       & CS$\_$Phds      \\
\midrule
GCOPE       & \textbf{60.44$\pm$2.06} & \textbf{50.77$\pm$5.38} & 56.67$\pm$1.67 & \textbf{15.84$\pm$2.65} & \textbf{\textcolor{red!80!black}{64.78$\pm$6.99}} & \textbf{\textcolor{red!80!black}{21.63$\pm$2.38}} & \textbf{55.91$\pm$7.32} & \textbf{\textcolor{red!80!black}{55.04$\pm$5.59}} & 25.74$\pm$2.02 \\
MDGPT       & \textcolor{red}{60.53$\pm$4.89} & \textcolor{red}{54.29$\pm$6.52} & \textbf{58.41$\pm$7.05} & 9.83$\pm$3.17 & \textbf{41.18$\pm$6.58} & \textbf{19.67$\pm$4.09} & 50.89$\pm$9.23 & \textbf{38.23$\pm$7.64} & \textcolor{red}{26.16$\pm$2.25} \\
MDGFM       & \textbf{\textcolor{red!80!black}{64.31$\pm$4.11}} & \textbf{\textcolor{red!80!black}{55.28$\pm$4.37}} & \textbf{\textcolor{red!80!black}{64.99$\pm$5.11}} & \textbf{16.10$\pm$2.13} & \textcolor{red}{55.80$\pm$6.06} & \textcolor{red}{20.10$\pm$1.75} & \textcolor{red}{59.27$\pm$7.68} & OOM & \textbf{25.95$\pm$2.46} \\
SAMGPT      & 35.88$\pm$6.95 & 50.26$\pm$5.89 & 55.08$\pm$9.11 & 15.23$\pm$4.77 & \textbf{48.68$\pm$6.70} & 17.91$\pm$6.14 & \textbf{\textcolor{red!80!black}{65.49$\pm$9.55}} & \textcolor{red}{43.09$\pm$7.45} & \textbf{26.17$\pm$6.30} \\
GraphGluing & 47.55$\pm$5.78 & 39.55$\pm$2.65 & \textcolor{red}{64.52$\pm$1.73} & \textcolor{red}{18.90$\pm$3.41} & 19.70$\pm$1.79 & 9.59$\pm$2.37 & 49.27$\pm$4.79 & OOM & 15.47$\pm$3.99 \\
SA2GFM      & 49.24$\pm$6.34 & 35.41$\pm$4.19 & 50.63$\pm$9.67 & \textbf{\textcolor{red!80!black}{19.24$\pm$5.00}} & 17.07$\pm$4.52 & 13.52$\pm$2.57 & 48.92$\pm$7.84 & OOM & \textbf{\textcolor{red!80!black}{26.63$\pm$1.67}} \\
\bottomrule
\end{tabular}
}
\end{table}

\paragraph{GFM label elasticity and feasibility.}
We next analyze how GFMs respond to additional supervision under different geometric regimes. 
The raw 1-shot and 5-shot results are reported in Table~\ref{tab:app-gfm-fewshot}. 
For model $M$ and regime $r$, we define label elasticity as
\begin{equation}
\label{eq:label-elasticity}
    \operatorname{Elasticity}(M,r)
    =
    \frac{1}{|\mathcal D_r(M)|}
    \sum_{G\in\mathcal D_r(M)}
    \left(
        S^{5\text{-shot}}_{M,G}
        -
        S^{1\text{-shot}}_{M,G}
    \right),
\end{equation}
where $\mathcal D_r(M)$ contains the datasets in regime $r$ on which $M$ is feasible. 
This quantity measures how much additional supervision can compensate for the inductive bias transferred from pretraining. 
High elasticity indicates that a model is mainly supervision-limited in that regime, whereas low or negative elasticity suggests that the bottleneck lies beyond label scarcity, such as geometric mismatch, optimization difficulty, or scalability constraints.

From Table~\ref{tab:app-gfm-fewshot}, averaged over GFMs, the elasticity is $18.41$ in the near-zero regime, $6.47$ in the positive regime, and $6.97$ in the negative regime. 
This indicates that near-zero citation-like graphs are highly label-elastic: once additional labels are provided, most GFMs improve substantially. 
In contrast, positive and negative regimes exhibit much weaker elasticity, suggesting that additional labels alone do not fully resolve the difficulty of these graphs. 
Their performance is instead shaped by a three-way interaction among geometry, task objective, and the transfer mechanism encoded by each GFM.

At the model level, Table~\ref{tab:app-gfm-fewshot} shows that GCOPE has the strongest average elasticity among complete-coverage methods, improving by $16.59$ points on average while remaining feasible on all nine natural graphs. 
MDGFM has comparable elasticity ($13.55$) but fails on Telecom, reducing its complete-coverage reliability. 
GraphGluing is especially revealing: it gains $29.18$ points in the near-zero regime but only $0.79$ in the positive regime and slightly decreases in the negative regime. 
This suggests a geometry--supervision--scalability frontier. 
A GFM may benefit strongly from additional labels when its transferred representation is structurally compatible with the target regime, but the same mechanism may fail to produce gains when the regime requires different geometric reasoning or exceeds the model's feasible operating range.

\paragraph{Coverage-aware GFM comparison.}
Table~\ref{tab:app-gfm-fewshot} also shows that several GFMs encounter OOM on Telecom, making it necessary to distinguish accuracy from feasibility. 
To make the effect of OOM explicit, we compare available-case averages with coverage-aware averages. 
The available-case average ignores OOM entries, whereas the coverage-aware average divides the total score by all nine datasets, thereby penalizing infeasible evaluations:
\begin{equation}
\label{eq:coverage-aware-score}
    S_{\mathrm{cov}}(M)
    =
    \frac{1}{9}
    \sum_{G}
    \mathbbm{1}[M \ \mathrm{is\ feasible\ on}\ G] \cdot S_{M,G}.
\end{equation}
Under 1-shot available-case averaging, MDGFM obtains the highest mean score ($31.68$). 
However, after coverage adjustment, MDGPT becomes the strongest 1-shot model ($31.14$), because it is feasible on all datasets. 
Under 5-shot coverage-aware evaluation, GCOPE becomes the strongest model ($45.20$), slightly ahead of MDGFM's coverage-aware score ($40.20$). 
This reversal shows that reporting only available-case performance overestimates specialized or memory-intensive GFMs. 
Feasibility is therefore not an implementation artifact; it is part of the model's practical transfer behavior.

\paragraph{Observation G.3: GFM progress should be measured by accuracy, elasticity, and coverage jointly.}
The GFM analysis based on Table~\ref{tab:app-gfm-fewshot} suggests that a single few-shot score is insufficient. 
A strong GFM should not only achieve high accuracy when it runs, but should also improve with additional labels and remain feasible across structurally difficult regimes. 
In this sense, GCOPE is a robust generalist, MDGFM is a high-performing but less coverage-stable model, and GraphGluing is a geometry-sensitive specialist. 
This provides a more nuanced evaluation than a flat GFM leaderboard, where OOM cases and regime-specific label elasticity are hidden.

\begin{table}[h]
\centering
\caption{Macro-F1 results on table Node Classification (NC) task.}
\label{tab:app-table-nc-results}
\resizebox{0.65\linewidth}{!}{
\begin{tabular}{c|ccc|cc}
\toprule
Baselines & Carcinogenesis & Hepatitis & PTE & Toxicology & F1 \\
\midrule
GCN       & 46.36$\pm$1.90     & \textcolor{red}{81.89$\pm$0.42}    & 79.63$\pm$1.79    & \textbf{41.14$\pm$9.10}   & 0.48$\pm$0.27  \\
GAT       & \textcolor{red}{56.81$\pm$4.51}     & 77.74$\pm$1.40    & 78.30$\pm$3.12     & \textcolor{red}{44.65$\pm$0.87}  & 0.22$\pm$0.01   \\
GraphSAGE & \textbf{\textcolor{red!80!black}{64.31$\pm$1.18}}   & \textbf{80.59$\pm$1.24}  & \textcolor{red}{81.38$\pm$0.08}   & \textbf{\textcolor{red!80!black}{49.70$\pm$3.74}} & 0.23$\pm$0.05 \\
MLP       & 48.87$\pm$2.40     & 66.98$\pm$4.50    & 78.98$\pm$0.90     & 35.51$\pm$0.00  & 0.63$\pm$0.27  \\
PCNet     & \textbf{52.82$\pm$2.02}   & \textbf{\textcolor{red!80!black}{82.86$\pm$2.13}}  & \textbf{80.47$\pm$0.02} & 35.94$\pm$1.25 & 0.20$\pm$0.01 \\
\midrule
HGNN      & 47.71$\pm$5.58   & 60.91$\pm$0.74  & 76.76$\pm$1.39   & 34.77$\pm$0.91 & 0.73$\pm$0.21 \\
HAT       & 30.19$\pm$0.51    & 37.63$\pm$1.16   & \textbf{\textcolor{red!80!black}{85.62$\pm$3.02}}    & 39.64$\pm$4.31  & \textbf{\textcolor{red!80!black}{2.74$\pm$0.29}}  \\
HGCNN     & 47.14$\pm$2.19   & 42.92$\pm$1.90  & 62.58$\pm$4.65   & 35.90$\pm$3.05 & 0.79$\pm$0.26 \\
HyboNet & 43.61$\pm$1.77     & 64.50$\pm$4.03    & 30.23$\pm$0.00     & 31.00$\pm$0.00   & \textbf{1.20$\pm$0.16}   \\
\midrule
CUSP      & 43.92$\pm$8.79    & 76.06$\pm$0.51    & 44.16$\pm$11.13    & 36.93$\pm$1.81   & 0.75$\pm$0.21   \\
QGCN      & 49.29$\pm$9.77   & 58.88$\pm$10.13 & 36.90$\pm$1.81   & 39.33$\pm$2.64 & \textcolor{red}{1.26$\pm$0.60} \\
\midrule
GraphMoRE & 54.06$\pm$4.79   & 79.67$\pm$2.05  & 78.29$\pm$1.06   & 36.09$\pm$1.22 & 0.74$\pm$0.12 \\
\bottomrule
\end{tabular}
}
\end{table}

\paragraph{Specialist--robustness behavior on table-derived graphs.}
Finally, we examine table-derived NC results in Table~\ref{tab:app-table-nc-results} through rank mean and rank volatility. 
For each model, we compute the mean rank and standard deviation of ranks across the five table-derived datasets. 
GraphSAGE has the best mean rank ($3.4$), wins two datasets, and appears in the top-3 on four out of five datasets. 
This makes it the strongest robust baseline on table-derived graphs. 
In contrast, HAT has the highest rank volatility ($5.61$): it wins PTE and F1, but ranks last on Carcinogenesis and Hepatitis. 
PCNet is another specialist, winning Hepatitis but ranking last on F1. 
GraphMoRE has a moderate mean rank ($5.4$) but low rank volatility ($1.82$), indicating stable but rarely dominant behavior.

When restricting Table~\ref{tab:app-table-nc-results} to the four medical table-derived datasets, the conclusion becomes even sharper. 
GraphSAGE achieves the best medical mean performance ($68.99$) and the best average medical rank ($1.75$), while HAT drops to an average medical rank of $7.25$ despite winning PTE. 
Therefore, HAT's strong overall impression is heavily driven by dataset-specific wins rather than broad robustness. 
This is consistent with the curvature-skewness view: table-derived graphs can contain strong tail signals, but different schema-induced tails favor different inductive biases.

\paragraph{Observation G.4: Table-derived graphs require reporting both winner counts and rank volatility.}
The results in Table~\ref{tab:app-table-nc-results} show that the model with the most striking wins is not necessarily the most reliable model. 
HAT demonstrates high upside but poor stability, while GraphSAGE provides the most consistent performance across medical table-derived datasets. 
This supports the central claim of CURVBENCH from a different angle: benchmark conclusions should not be reduced to a single averaged score. 
For heterogeneous relational data, robust evaluation must distinguish between generalists, specialists, and methods whose success is concentrated in narrow structural conditions.

\section{Efficiency analysis}
We further analyze the computational efficiency of different model families on the NC task. 
As shown in figure~\ref{fig:left} and figure~\ref{fig:right}, we visualize the runtime footprint through log-scale heatmaps and a train--test efficiency landscape. 
Before visualization, we correct two apparent logging outliers: the training time of HyboNet on PubMed is adjusted from $2278.29$ to $227.83$, and the training time of QGCN on Cornell is adjusted from $10472.98$ to $104.73$. 
These two entries are several orders of magnitude larger than neighboring runs and would otherwise dominate the visualization.

\begin{figure}[t]
\centering

\begin{minipage}{0.48\linewidth}
    \centering
    \includegraphics[width=\linewidth]{"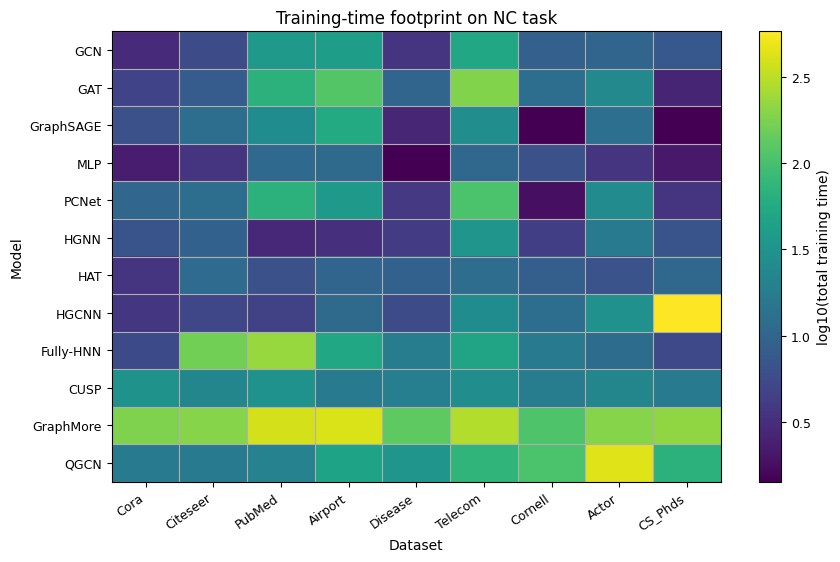"}
    \caption{Total training time heatmap on Node Classification (NC) task.}
    \label{fig:left}
\end{minipage}
\hfill
\begin{minipage}{0.48\linewidth}
    \centering
    \includegraphics[width=\linewidth]{"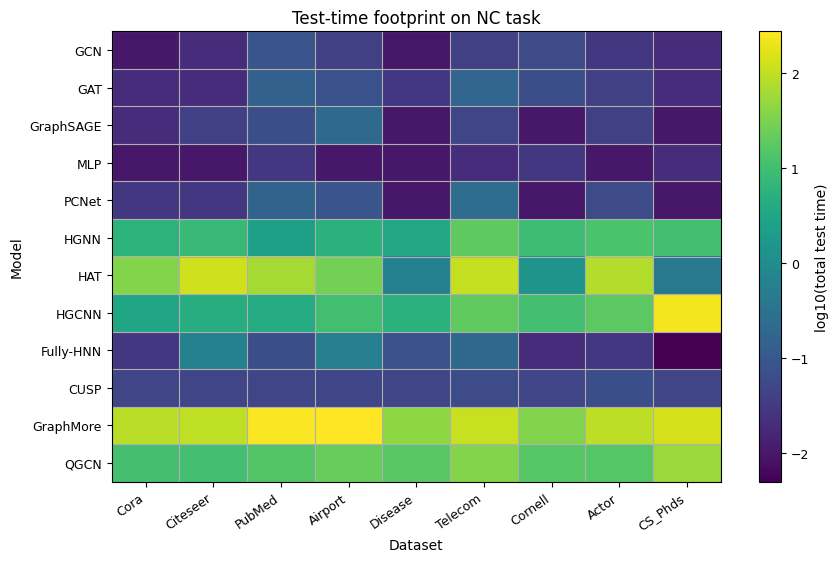"}
    \caption{Toatl test time heatmap on Node Classification (NC) task.}
    \label{fig:right}
\end{minipage}

\end{figure}

The training-time heatmap shows that Euclidean and feature-only methods are generally the most efficient. 
MLP has the lowest median training time ($3.59$), followed by HGNN ($6.71$) and GCN ($9.05$). 
Among standard message-passing baselines, GCN, GraphSAGE, GAT, and PCNet remain within a similar order of magnitude, with median training times between $9.05$ and $12.53$. 
In contrast, GraphMoRE is substantially more expensive, with a median training time of $197.20$, reflecting the overhead of adaptive Riemannian experts. 
QGCN also incurs a high median training cost ($46.22$), consistent with the additional computation required by pseudo-Riemannian geometry.

Figure~\ref{fig:efficiency-mirrored} reveals an even sharper separation. 
MLP, GCN, GAT, GraphSAGE, PCNet, HyboNet, and CUSP all remain below $0.1$ median test time, indicating that they are efficient at inference once trained. 
However, several non-Euclidean models introduce substantial inference overhead. 
HGNN, HGCNN, QGCN, HAT, and GraphMoRE have median test times of $7.92$, $9.97$, $16.13$, $37.12$, and $97.07$, respectively. 
This suggests that the computational burden of non-Euclidean modeling is not limited to optimization; for some architectures, it persists during inference due to manifold operations, geometry-specific transformations, or expert routing.

The train--test efficiency landscape further separates models into three groups. 
First, MLP and standard Euclidean GNNs occupy the low-cost region, making them strong efficiency baselines. 
Second, CUSP and HyboNet have moderate training cost but low test-time cost, suggesting that some geometric or spectral overhead can be amortized after training. 
Third, QGCN, HAT, HGCNN, and GraphMoRE are high-overhead methods, especially at inference. 
GraphMoRE is the most expensive model overall, combining the largest median training time with the largest median test time. 
Therefore, efficiency should be considered jointly with performance: geometry-aware models may improve robustness or regime alignment, but their practical value depends on whether the performance gain justifies the additional train-time and inference-time cost.

\begin{wrapfigure}{r}{0.52\linewidth}
    \centering
    \vspace{-0.5em}
    \includegraphics[width=\linewidth]{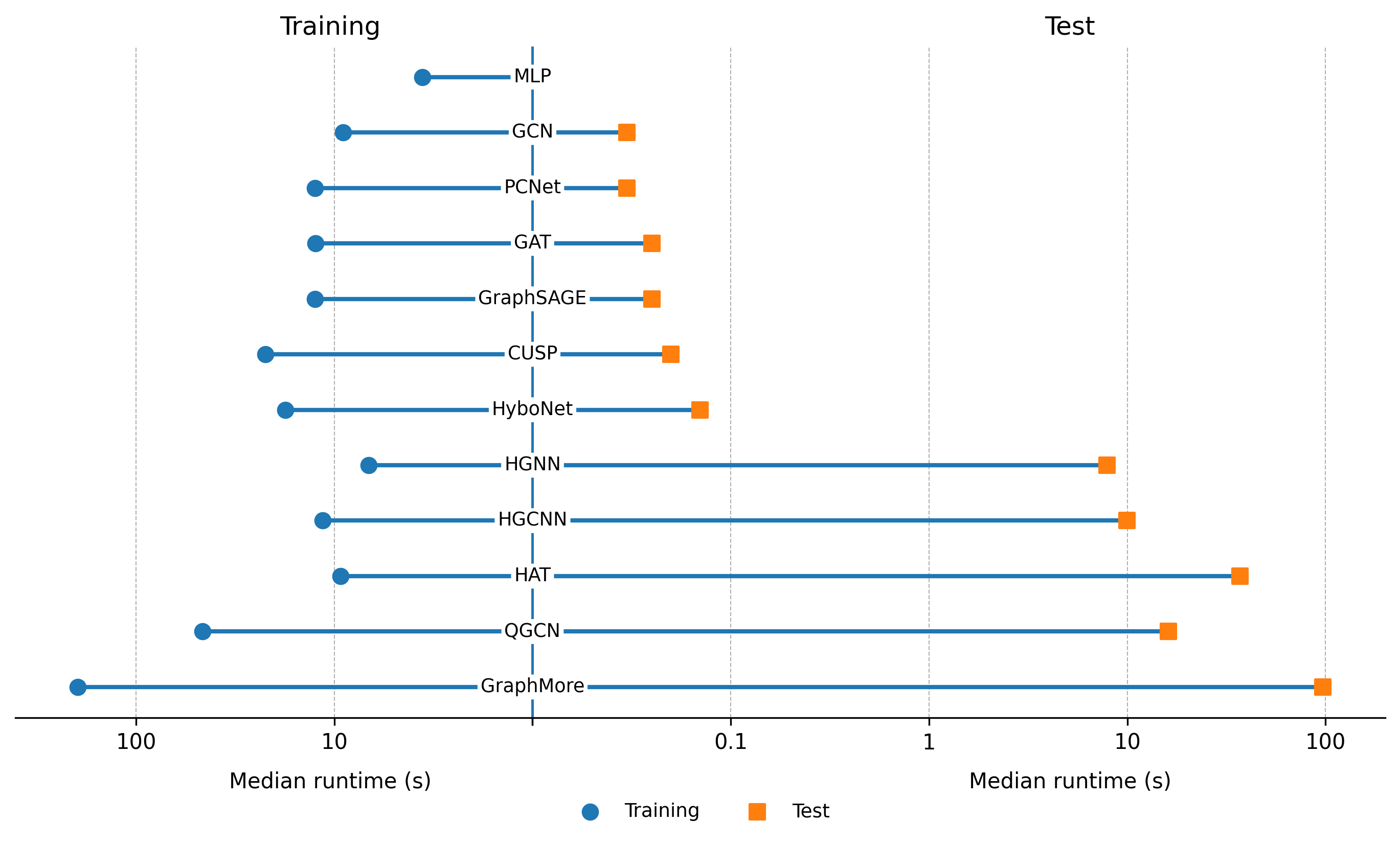}
    \caption{Mirrored efficiency diagram across models.}
    \label{fig:efficiency-mirrored}
    \vspace{-1em}
\end{wrapfigure}

\paragraph{Observation H: Geometry-aware modeling introduces distinct efficiency profiles.}
The efficiency results show that computational cost is itself geometry-dependent. 
Euclidean baselines are consistently lightweight, while adaptive or manifold-heavy methods often incur substantial overhead. 
Importantly, training cost and inference cost do not always move together. 
For example, HyboNet has nontrivial training cost but remains relatively cheap at inference, whereas HAT and GraphMoRE are expensive during testing. 
This distinction matters for deployment: a method that is acceptable for offline training may still be impractical for repeated inference. CURVBENCH evaluates not only whether a model is accurate under a curvature regime, but also whether its geometric inductive bias is computationally affordable.

\newpage

\end{document}